\documentclass[a4paper,fleqn]{cas-sc}

\usepackage[authoryear,longnamesfirst]{natbib}

\usepackage[ruled]{algorithm2e}
\usepackage{appendix}
\usepackage{upgreek}
\usepackage{multirow}
\usepackage{amsfonts}
\usepackage{amsmath}
\usepackage{amsthm}
\usepackage{amssymb}
\usepackage{array}
\usepackage{bm}
\usepackage{caption}
\usepackage[color]{changebar}
\usepackage{datetime}
\usepackage{enumerate}
\usepackage{enumitem}
\usepackage{graphicx}
\usepackage[table]{xcolor}
\usepackage{booktabs}
\usepackage{caption}
\usepackage{subcaption}

\newcommand{\f}[0]{\textbf{f}}
\newcommand{\g}[0]{\textbf{g}}

\renewcommand{\alpha}{\upalpha}
\renewcommand{\beta}{\upbeta}
\renewcommand{\sigma}{\upsigma}
\renewcommand{\gamma}{\upgamma}
\renewcommand{\phi}{\upphi}
\renewcommand{\mu}{\upmu}
\renewcommand{\epsilon}{\upepsilon}
\renewcommand{\kappa}{\upkappa}
\renewcommand{\xi}{\upxi}
\renewcommand{\eta}{\upeta}
\renewcommand{\lambda}{\uplambda}
\renewcommand{\rho}{\uprho}
\renewcommand{\delta}{\updelta}
\renewcommand{\theta}{\uptheta}
\renewcommand{\upsilon}{\upupsilon}
\renewcommand{\vartheta}{\upvartheta}
\renewcommand{\nu}{\upnu}
\renewcommand{\omega}{\upomega}
\renewcommand{\zeta}{\upzeta}
\renewcommand{\varepsilon}{\upvarepsilon}
\renewcommand{\varphi}{\upvarphi}
\renewcommand{\psi}{\uppsi}

\def\tsc#1{\csdef{#1}{\textsc{\lowercase{#1}}\xspace}}
\tsc{WGM}
\tsc{QE}

\usepackage{algorithmic}

\newcommand{\STATEx}{\item[]}

\newtheorem{theorem}{Theorem}
\newtheorem{remark}{Remark}
\newtheorem{corollary}{Corrolary}

\begin{document}
\let\WriteBookmarks\relax
\def\floatpagepagefraction{1}
\def\textpagefraction{.001}

\shorttitle{}    

\shortauthors{}  

\title [mode = title]{Multi-output Gaussian process prediction of physical fields under linear equality constraints}  



%

\author[1,3]{Mahamat Hamdan Nassouradine} 

\cormark[1]


\ead{mahamat-hamdan.nassouradine@cea.fr}



\affiliation[1]{organization={Université Paris-Saclay, CEA, Service de Génie Logiciel pour la Simulation},
            city={Gif-sur-Yvette},
            postcode={91191}, 
            country={France}}

\author[1]{Clément Gauchy}


\ead{clement.gauchy@cea.fr}



\author[2]{Pierre-Emmanuel Angeli}
\ead{pierre-emmanuel.anglei@cea.fr}
\author[3]{Sébastien {Da Veiga}}
\ead{sebastien.da-veiga@ensai.fr}
\affiliation[2]{organization={Université Paris-Saclay, CEA, Service de Thermohydraulique et de Mécanique des Fluides},
            city={Gif-sur-Yvette},
            postcode={91191}, 
            country={France}}
\affiliation[3]{organization={Univ Rennes, Ensai, CNRS, CREST - UMR 9194},
            city={Rennes},
            postcode={F-35000}, 
            country={France}}

\cortext[1]{Corresponding author}



\begin{abstract}
We address the simultaneous prediction of multiple high-dimensional physical fields governed by linear equality constraints, a setting that arises in many real-world applications in physics machine learning. Gaussian process (GP) regression is a widely used surrogate modeling approach due to its effectiveness in small-sample regimes and its ability to provide uncertainty quantification. However, applying GP models in this setting raises two major challenges:
the high dimensionality of the discretized output fields and the enforcement of the physical constraint in predictions. For the latter, a common strategy consists in deducing one output from the others via the constraint relation. Through an empirical benchmark, we show that this deductive approach is sensitive to the arbitrary choice of which output to deduce, affecting both predictive accuracy and uncertainty quantification. Consequently, there is a need for an approach that treats all fields symmetrically while strictly respecting the underlying physics. Motivated by these limitations, we propose a robust framework for jointly modeling constrained multi-field data. Our approach first leverages a specific PCA procedure for multi-field data, coined row-wise PCA, which has the interesting property of preserving the constraint in the latent space. Since standard PCA strategies for multi-field data (field-wise, column-wise) do not preserve such constraints, we investigate theoretically the optimality of the row-wise choice and provide conditions under which it incurs a negligible reconstruction cost. In a second step, we consider a linearly-constrained multi-output GP approach based on a specific kernel parametrization which is trained on the latent space of row-wise PCA. The proposed framework is validated on a population dynamics problem and on an industrial computational fluid dynamics application, which involves the prediction of Reynolds stress tensor components under the incompressibility constraint. We demonstrate competitive predictive performance while guaranteeing strict constraint satisfaction.
\end{abstract}



\begin{keywords}
Metamodeling \sep linear equality constraints \sep physical fields \sep Gaussian process \sep PCA
\end{keywords}

\maketitle

\section{Introdution}
\label{sec:intro}
Metamodeling is now a fundamental approach for approximating computationally expensive numerical simulations or costly experiments \citep{OHagan2006, Kennedy2000, Santner2003}. This technique is essential for various engineering applications, such as structural optimization \citep{Jansson2003}, aeronautics \citep{SobieszczanskiHaftka1997}, multifidelity optimization \citep{Forrester2007, Forrester2008}, sensitivity analysis \citep{Nanty2016}, or reliability analysis \citep{SuPengHu2017, LiSadoughi2020}. In this work, we specifically focus on building a metamodel that simultaneously predicts multiple spatial or temporal fields governed by linear equality constraints, which are ubiquitous in physics-based simulations. Gaussian process (GP) regression, also known as kriging in geostatistics \citep{Matheron1963}, has established itself as a method of choice in metamodeling context, thanks to its flexibility in capturing nonlinear relationships, its principled quantification of predictive uncertainty, and its effectiveness in small-sample regimes \citep{RasmussenWilliams2005}. This approach has been successfully applied to approximate scalar outputs across various domains: uncertainty quantification \citep{OHagan2006}, thermohydraulic studies in nuclear safety \citep{MarrelChabridon2022}, or magnetic field reconstruction in physics \citep{MacedoCastro2010}. To capture dependencies across multiple scalar outputs, this framework has been naturally extended to Multi-Output Gaussian Processes (MOGP) using structured covariance models such as the linear model of coregionalization (LCM) \citep{GoulardVoltz1992, Rougier2008}, the intrinsic coregionalization model (ICM), or process convolutions \citep{Higdon2002, AlvarezRosascoLawrence2012}.

In many physical applications of interest, however, the output fields are subject to linear equality constraints: for instance, an incompressibility condition on the velocity field, or a conservation law imposing that species mass fractions sum to unity at every point. Accounting for these constraints is crucial: first, it guarantees physically consistent predictions. Second, we expect that accounting for the constraints will produce a more accurate metamodel.  When such constraints are of the form $\sum_{j=1}^Q \alpha_j(\mathbf{x}) f_j(\mathbf{x}) = c(\mathbf{x})$, a seemingly straightforward strategy consists in exploiting the constraint algebraically: one output is deduced from the others via the constraint relation, and a standard (unconstrained) GP surrogate is built only on the remaining $Q-1$ fields. We refer to this as the deductive approach. It is appealing because it allows the direct reuse of any existing unconstrained framework and at first sight guarantees constraint satisfaction on the predictive mean by construction. However, this approach may suffer from a fundamental structural limitation: the treatment of the outputs becomes arbitrary, since the deduced field is not modeled jointly with the others but is instead recovered a~posteriori through the constraint relation. As a consequence, the deduced output accumulates the prediction errors of all modeled fields, and its predictive quality depends on the arbitrary choice of the deduced output. This phenomenon is closely related to the well-known dependence on the reference category in the additive log-ratio  transformation of compositional data analysis \citep{Aitchison1986}, where the statistical properties of the transformed variables are sensitive to the choice of the component used as denominator.

The structural limitations of this deductive approach motivate the search for methods that integrate constraints directly into the GP model. Several methods have been developed for integrating constraints into
GP regression~\citep{Swiler2020}. For inequality constraints, approaches include data augmentation \citep{AbrahamsenBenth2001}, finite-dimensional projections \citep{DaVeigaMarrel2020, MaatoukBay2017}, and shape constraint methods \citep{WangBerger2016}. For equality constraints, early work on linear forms dates back to \citet{Graepel2003}, leading to more general approaches that encode constraints directly in the kernel by restricting the GP to the solution space of a linear operator \citep{Jidling2017, LangeHegermann2018}. The parametrization approach of \citet{Jidling2017}, in particular, constructs a covariance kernel such that every sample path of the GP satisfies the constraint over the entire domain, thus ensuring constraint satisfaction on both the predictive mean and posterior samples. More recently, \citet{Pilar2023} specifically addressed linear and non linear constraints in multitask GPs. However, these constrained GP methods have been developed and demonstrated for problems with low-dimensional outputs, and they inherit the cubic scaling of full GP inference $\mathcal{O}(N^3 P^3)$ or $\mathcal{O}(N^3 + P^3)$ in favorable cases \citep{Stegle2011}, where $N$ is the sample size and $P$ the output dimension, which precludes their direct application to high-dimensional discretized fields.

To overcome the dimensional limitation in general case, a widely adopted strategy in the GP literature combines a dimensionality reduction step with GP regression in the resulting low-dimensional latent space. The seminal contribution of \citet{Higdon2008} proposed applying Principal Component Analysis (PCA) to the output fields and training independent GPs on the principal component scores. More recently, \citet{Xing2021} extended this framework by independently reducing the dimension of each individual field and coupling these representations with MOGP through the LCM kernel, enabling the joint modeling of several correlated fields in high dimensions. Other strategies have explored nonlinear alternatives such as Isomap-GP or KPCA-GP \citep{Xing2015, ZhouPeng2020}, developed extended formulations such as the extended LCM \citep{Wang2022ELMC} that use invertible neural networks to capture spatial nonlinearities, or introduced higher-order GPs \citep{Zhe2019} exploiting Kronecker product properties to further reduce computational cost. We also mention the PCA-GP emulators of \citet{Herman2020} for predicting high-dimensional microstructure morphology fields, and the efficient surrogate of \citet{Mak2018} for large eddy simulations, which also relies on dimensionality reduction combined with GP regression. While all these methods successfully address the dimensionality challenge, they do not naturally guarantee compliance with linear equality constraints unless they are embedded within a deductive approach, which, as discussed above, introduces its own structural limitations.

In this paper, we propose an approach that addresses both the dimensionality and the constraint challenges in a unified framework. The key idea is to perform dimensionality reduction in a way that  {preserves the constraint structure} in the latent space, so that constrained GP regression can then be applied in this reduced space. We focus on PCA as the dimensionality reduction tool, since nonlinear alternatives are both less robust in small-sample regimes and incompatible with the linear constraint structure. For multi-field data, PCA admits several formulations depending on how the component fields are arranged before decomposition. In the geoscience and spatio-temporal statistics literature, this extension of PCA to vector fields is known as multivariate Empirical Orthogonal Functions (EOFs) \citep{Preisendorfer1988, Hannachi2007, Sparnocchia2003}, with applications in climatology \citep{AlveraAzcarate2007}, oceanography \citep{Liang2018}, and ecology \citep{Petitgas2014}. Two main global formulations exist \citep{Preisendorfer1988}: the  {time-modulation form} (TMF, or column-wise PCA), which concatenates the fields column-wise and captures inter-field dependencies in the spatial dimension, and the  {space-modulation form} (SMF, or row-wise PCA), which stacks the fields row-wise and produces a shared spatial basis. In contrast to these global methods, a third option consists in applying the reduction independently to each component. Following \citet{Xing2021}, we formally refer to this strategy as {field-wise} PCA. Theoretical underpinnings for such decoupled representations of vector fields can also be found in functional data analysis \citep{HappGreven2018} and generalized Karhunen--Lo{è}ve expansions \citep{Perrin2013}. Among these strategies, the row-wise approach has a distinctive property: it preserves linear equality constraints exactly in the latent space, regardless of the truncation level. This makes it uniquely suited for coupling with a constrained MOGP model. We propose the row-wise PCA-Constrained MOGP framework (coined \texttt{Row-CMO}), which combines the row-wise PCA with constrained MOGP regression based on the parametrization approach of \citet{Jidling2017}, ensuring that both the predictive mean and posterior samples satisfy the physical constraint to machine precision.
The column-wise and field-wise PCA strategies, which do not preserve the constraint after truncation, can only be used within a deductive approach for constrained problems. A natural question then arises: what is the reconstruction cost of choosing row-wise PCA over these alternatives ? To address this, we provide a theoretical analysis comparing the three strategies from the perspective of reconstruction error. 
The main contributions of this paper are the following:
\begin{enumerate}
    \item We propose the \texttt{Row-CMO} framework for predicting multi-field outputs under linear equality constraints.

    \item Through empirical benchmarking, we demonstrate the inherent heterogeneity of deductive approaches, showing that performance can degrade significantly depending on the arbitrary choice of the deduced output. 

    \item We provide a theoretical analysis of multi-field PCA strategies (field-wise, column-wise, and row-wise) by bounding the excess reconstruction error associated with global methods. Based on classical perturbation theory, our analysis identifies inter-field spectral heterogeneity as the sole factor driving this discrepancy.
    \item We validate \texttt{Row-CMO} on two case studies: the Lotka--Volterra population dynamics system \citep{Wangersky1978}, subject to a sum-to-zero conservation constraint, and a Computational Fluid Dynamics (CFD) application predicting Reynolds stress tensor components on the Buice--Eaton 2D diffuser \citep{BuiceEaton2000}, governed by the incompressibility constraint.
\end{enumerate}
The paper is organized as follows. Section~\ref{sec:problem} establishes the problem statement and recalls the necessary background on Gaussian process regression. Building on this framework, Section~\ref{sec:constraint} exposes the vulnerabilities of the deductive approach to constraint enforcement. The core theoretical analysis of the different multi-field PCA strategies is then presented in Section~\ref{sec:pca_theory}. These theoretical insights and the proposed model are subsequently validated through numerical experiments in Section~\ref{sec:experiments}. 
All the numerical experiments, except for the CFD application, are available at \url{https://github.com/nasrmht/paper_benchmark_repo}.


\section{Problem statement}\label{sec:problem}

\subsection{Problem formulation}\label{sec:formulation}

We consider the output of a parameterized numerical simulator  that produces $Q$ coupled scalar fields, each evaluated over a fixed set of $S$ discretes coordinates (e.g spatial locations or time steps). The simulator is governed by a vector of input parameters $\mathbf{x} \in \mathcal{X} \subset \mathbb{R}^D$, which may represent physical constants, boundary conditions, or model closure coefficients. For a given input $\mathbf{x}^{(i)}$, the simulator returns $Q$ discretized outputs $\mathbf{y}_1^{(i)}, \ldots, \mathbf{y}_Q^{(i)}$, where each $\mathbf{y}_j^{(i)} = (u_j(\nu_1, \mathbf{x}^{(i)}), \ldots, u_j(\nu_S, \mathbf{x}^{(i)}))^\top \in \mathbb{R}^S$ collects the values of the $j$-th scalar field $u_j$ evaluated at the fixed coordinates  $\nu_1, \ldots, \nu_S$. We denote by $\mathbf{y}^{(i)} = [\mathbf{y}^{{(i)}\top}_1, \ldots, \mathbf{y}^{{(i)}\top}_Q]^\top \in \mathbb{R}^P$, with $P = Q \times S$, the concatenated output vector.
We assume that a training set of $N$ simulation runs is available:
$  \bigl(\mathbf{x}^{(i)},\, \mathbf{y}^{(i)} = [\mathbf{y}^{(i)\top}_1, \ldots, \mathbf{y}^{(i)\top}_Q]^\top\bigr)_{1 \leq i \leq N},
$
where $N$ is typically small (on the order of tens to hundreds) due to the computational cost of each simulation. 
The data are assumed noise-free,
as is the case for outputs of deterministic numerical simulators.  Our goal is to approximate the mapping
$\mathbf{f} = [\mathbf{f}_1^\top, \ldots, \mathbf{f}_Q^{\top}]^\top:
\mathcal{X} \to \mathbb{R}^P$ such that
\begin{equation}
\mathbf{y}^{(i)} = \mathbf{f}(\mathbf{x}^{(i)}),  
\label{eq:model}
\end{equation}
$i = 1, \ldots, N$.
We further assume that the fields satisfy a linear equality constraint of the form
\begin{equation}\label{eq:constraint}
  \mathcal{F}[\mathbf{f}(\mathbf{x})] = \sum_{j=1}^Q \alpha_j(\mathbf{x}) \, \mathbf{f}_j(\mathbf{x}) = \mathbf{c}(\mathbf{x})
\end{equation}
for all $\mathbf{x} \in \mathcal{X}$, where $\alpha_1(\cdot), \ldots, \alpha_Q(\cdot)$ are known, non-zero input-dependent coefficients and $\mathbf{c} : \mathcal{X} \to \mathbb{R}^S$ is a known constraint function.
For example, for two fields constrained to sum to 1 on each discrete coordinate, we have $\alpha_1(\mathbf{x})=\alpha_2(\mathbf{x})=1$ and  $\mathbf{c}=1_{\mathbb{R}^S}$ for all $\mathbf{x}\in\mathcal{X}$.
We focus first on the case where the constraint coefficients $\alpha_j$ are {independent} of the input $\mathbf{x}$ and the second member is zero ($\mathbf{c}(\mathbf{x}) = \mathbf{0}$). Under these assumptions, the constraint reads
\begin{equation}\label{eq:constraint_zero}
  \sum_{j=1}^Q \alpha_j \, \mathbf{f}_j(\mathbf{x}) = \mathbf{0}_S, \quad \forall \, \mathbf{x} \in \mathcal{X}.
\end{equation}
The generalizations to input-dependent coefficients $\alpha_j(\mathbf{x})$ and non-zero second member $\mathbf{c}(\mathbf{x}) \neq \mathbf{0}$ are discussed in Section~\ref{sec:generalization}, while the
extension to noisy observations is discussed in
Section~\ref{sec:noisy_data}.
\subsection{Gaussian process regression}\label{sec:gp}

We adopt Gaussian process regression as our modeling tool for functions $\f_1,\ldots, \f_Q$, given its well-established effectiveness in small-sample regimes and its principled uncertainty quantification \citep{RasmussenWilliams2005}. We briefly recall the single-output and multi-output formulations, which serve both as building blocks for our framework and as components of the competing methods in our benchmark.

\subsubsection{Single-output Gaussian process}\label{sec:sogp}

Given a training set $(\mathbf{x}^{(i)}, z^{(i)})_{1 \leq i \leq N} \subset \mathbb{R}^D \times \mathbb{R}$ with $z^{(i)} = f(\mathbf{x}^{(i)})$, a GP prior
$f(\mathbf{x}) \sim \mathcal{GP}(\mu(\mathbf{x}), k(\mathbf{x}, \mathbf{x}'))$ is placed on $f$. Conditioning on the observations, the posterior distribution at a test point $\mathbf{x}_*$ is Gaussian, $f_* \mid \mathbf{X}, \mathbf{z}, \mathbf{x}_* \sim \mathcal{N}\bigl(\bar f_*, \mathrm{Var}(f_*)\bigr)$, with
\begin{equation}
\bar f_* = \mu(\mathbf{x}_*) + \mathbf{k}_*^\top \mathbf{K}^{-1} (\mathbf{z} - \boldsymbol{\mu}), \qquad
\mathrm{Var}(f_*) = k(\mathbf{x}_*, \mathbf{x}_*) - \mathbf{k}_*^\top \mathbf{K}^{-1} \mathbf{k}_*,
\end{equation}
where $\mathbf{z} = [z_1, \ldots, z_N]^\top$,
$\boldsymbol{\mu} = [\mu(\mathbf{x}_1), \ldots, \mu(\mathbf{x}_N)]^\top$, $[\mathbf{K}]_{ij} = k(\mathbf{x}^{(i)}, \mathbf{x}^j)$, and $\mathbf{k}_* = [k(\mathbf{x}^1, \mathbf{x}_*), \ldots, k(\mathbf{x}^N,
\mathbf{x}_*)]^\top$. The kernel hyperparameters are estimated by maximizing the marginal log-likelihood
\citep{RasmussenWilliams2005}. In practice, a small nugget
$\sigma^2 \mathbf{I}_N$ is added to $\mathbf{K}$ for numerical stability.
\subsubsection{Multi-output Gaussian process}\label{sec:mogp}
 
Now consider the multi-output regression task of learning a vector-valued function $\mathbf{f}: \mathcal{X} \subset \mathbb{R}^D \to \mathbb{R}^Q$ from a training set $(\mathbf{x}^{(i)}, \mathbf{z}^{(i)} = [z_1^{(i)}, \ldots, z_Q^{(i)}]^\top)_{1 \leq i \leq N}$ with $z_q^{(i)} = f_q(\mathbf{x}^{(i)})$. A multi-output GP prior is placed on $\mathbf{f}$:
\[
  \mathbf{f}(\mathbf{x}) \sim \mathcal{GP}\bigl(\boldsymbol{\mu}(\mathbf{x}),\, \mathbf{k}(\mathbf{x}, \mathbf{x}')\bigr),
\]
where $\boldsymbol{\mu} : \mathcal{X} \to \mathbb{R}^Q$ is a vector-valued mean function and $\mathbf{k} : \mathcal{X} \times \mathcal{X} \to \mathbb{R}^{Q \times Q}$ is a matrix-valued covariance kernel with entries $[\mathbf{k}(\mathbf{x}, \mathbf{x}')]_{qq'} = \text{Cov}[f_q(\mathbf{x}), f_{q'}(\mathbf{x}')]$. 
Denoting by $\mathbf{Z} = [\mathbf{z}_1^\top, \ldots,
\mathbf{z}_N^\top]^\top \in \mathbb{R}^{NQ}$ the concatenated observations, the posterior at a test point $\mathbf{x}_*$ is Gaussian with mean and covariance
\begin{equation}
\begin{aligned}
    \bar{\mathbf{f}}_* &= \boldsymbol{\mu}(\mathbf{x}_*) + \mathbf{k}(\mathbf{X},
\mathbf{x}_*)^\top \mathbf{k}(\mathbf{X}, \mathbf{X})^{-1}
(\mathbf{Z} - \boldsymbol{\mu}(\mathbf{X})),\\
\mathrm{Cov}(\mathbf{f}_*) &= \mathbf{k}(\mathbf{x}_*, \mathbf{x}_*) -
\mathbf{k}(\mathbf{x}_*, \mathbf{X}) \mathbf{k}(\mathbf{X}, \mathbf{X})^{-1}
\mathbf{k}(\mathbf{X}, \mathbf{x}_*),
\end{aligned}
\end{equation}
where $\mathbf{k}(\mathbf{X}, \mathbf{X}) \in \mathbb{R}^{NQ \times NQ}$
is the block covariance matrix.  
A popular class of matrix-valued kernels is the linear model of coregionalization (LCM)~\citep{GoulardVoltz1992, Rougier2008}, which expresses $\mathbf{k}$ as a sum of separable terms:
\begin{equation}\label{eq:lcm}
  \mathbf{k}(\mathbf{x}, \mathbf{x}') = \sum_{r=1}^R k_r(\mathbf{x}, \mathbf{x}')\, \mathbf{B}_r,
\end{equation}
where each $k_r(\cdot, \cdot)$ is a scalar kernel and each $\mathbf{B}_r \in \mathbb{R}^{Q \times Q}$ is a positive semi-definite coregionalization matrix capturing inter-output dependencies at the $r$-th scale. The special case $R = 1$ reduces to the intrinsic coregionalization model (ICM): $\mathbf{k}(\mathbf{x}, \mathbf{x}') = k(\mathbf{x}, \mathbf{x}')\, \mathbf{B}$. Hyperparameters, including the kernel parameters and the entries of $\mathbf{B}_r$, are estimated by maximizing the marginal log-likelihood
\begin{equation}\label{eq:mll}
 \log p(\mathbf{Z} \mid \mathbf{X}, \boldsymbol{\theta}) = -\tfrac{1}{2}
\mathbf{Z}^\top \mathbf{k}(\mathbf{X}, \mathbf{X})^{-1} \mathbf{Z}
- \tfrac{1}{2} \log\bigl|\mathbf{k}(\mathbf{X}, \mathbf{X})\bigr|
- \tfrac{NQ}{2} \log 2\pi,
\end{equation}
using gradient-based optimization \citep{Nocedal2006}. For more details on the choice of kernel structure and discussion about 
multi-output GPs, 
we refer the reader to \citet{AlvarezRosascoLawrence2012}.
 
\subsection{Multi-output GP regression under linear equality constraints}\label{sec:constrained_gp}
 
We now recall how linear equality constraints can be enforced within the MOGP framework using the parametrization approach proposed by \citet{Jidling2017}. Consider a multi-output function $\mathbf{f} : \mathcal{X} \to \mathbb{R}^Q$ subject to the linear constraint
\begin{equation}\label{eq:gp_constraint}
  \mathcal{L}[\mathbf{f}](\mathbf{x}) = \sum_{j=1}^Q \alpha_j f_j(\mathbf{x}) = 0, \quad \forall\, \mathbf{x} \in \mathcal{X},
\end{equation}
which can be written in vector form as $\boldsymbol{\alpha}^\top \mathbf{f}(\mathbf{x}) = 0$ with $\boldsymbol{\alpha} = [\alpha_1, \ldots, \alpha_Q]^\top$. Considering a MOGP prior on $\mathbf{f}$, the problem amounts to imposing the linear constraint on a vector-valued Gaussian process $\mathbf{f} \sim \mathcal{GP}(\boldsymbol{\mu}_f(\mathbf{x}),\, \mathbf{k}_f(\mathbf{x}, \mathbf{x}'))$.
 
The parametrization approach constrains the search space of $\mathbf{f}$ to the set of functions satisfying \eqref{eq:gp_constraint} by expressing $\mathbf{f}$ as a linear transformation of an unconstrained function. Specifically, one assumes that $\mathbf{f}$ is related to another function $\mathbf{g} : \mathcal{X} \to \mathbb{R}^\ell$ through a linear operator $\mathcal{G}$:
\begin{equation}\label{eq:param_operator}
  \mathbf{f}(\mathbf{x}) = \mathcal{G}_{\mathbf{x}}[\mathbf{g}].
\end{equation}
With this formulation, the constraint on $\mathbf{f}$ becomes $\mathcal{L}[\mathcal{G}_{\mathbf{x}}[\mathbf{g}]] = 0$, and we require this to hold for any function $\mathbf{g}$, thereby transferring the restriction from $\mathbf{f}$ to the operator $\mathcal{G}$ itself. Since both $\mathcal{L}$ and $\mathcal{G}$ are linear, they can be represented by matrices in our setting. Let $\mathbf{C}_\mathcal{L} \in \mathbb{R}^{Q \times \ell}$ denote the matrix representation of $\mathcal{G}$. The relation \eqref{eq:param_operator} becomes
\begin{equation}\label{eq:param}
  \mathbf{f}(\mathbf{x}) = \mathbf{C}_\mathcal{L}\, \mathbf{g}(\mathbf{x})
\end{equation}
and the constraint~\eqref{eq:gp_constraint} holds for any $\mathbf{g}$ if and only if
\begin{equation}\label{eq:param_cond}
  \boldsymbol{\alpha}^\top \mathbf{C}_\mathcal{L} = \mathbf{0}_\ell^\top.
\end{equation}
Since Gaussian processes are stable under linear transformations, placing a GP prior $\mathbf{g} \sim \mathcal{GP}(\boldsymbol{\mu}_g,\, \mathbf{k}_g)$ on the unconstrained function, where $\mathbf{k}_g : \mathcal{X} \times \mathcal{X} \to \mathbb{R}^{\ell \times \ell}$ is a matrix-valued kernel, induces a valid GP prior on $\mathbf{f}$, with mean and covariance
\begin{equation}\label{eq:param_prior}
  \boldsymbol{\mu}_F(\mathbf{x}) = \mathbf{C}_\mathcal{L}\, \boldsymbol{\mu}_g(\mathbf{x}), \qquad
  \mathbf{k}_f(\mathbf{x}, \mathbf{x}') = \mathbf{C}_\mathcal{L}\, \mathbf{k}_g(\mathbf{x}, \mathbf{x}')\, \mathbf{C}_\mathcal{L}^\top.
\end{equation}
By construction, every realization of $\mathbf{f}$ lies in the column space of $\mathbf{C}_\mathcal{L}$, which is contained in the null space of $\boldsymbol{\alpha}^\top$, and therefore satisfies the constraint exactly. Crucially, this property is inherited by the posterior distribution: both the posterior mean and any sample drawn from the posterior of $\mathbf{f}$ satisfy~\eqref{eq:gp_constraint} at every point in $\mathcal{X}$.
 
\begin{remark}
The parametrization approach is presented here for the specific case of a single linear equality constraint with constant coefficients. It applies, however, in a more general framework involving multiple constraints defined by a matrix-valued linear operator $\mathcal{L}$. See \citet{LangeHegermann2018} for a systematic algorithmic construction of the parametrization using Gr\"{o}bner bases. In that setting, $\mathbf{C}_\mathcal{L}$ is chosen in the intersection of the null spaces of the rows of $\mathcal{L}$. 
\end{remark}
 
 \paragraph{Solution space and hyperparameter estimation}
The parametrization matrix $\mathbf{C}_\mathcal{L}$ is not unique: any matrix whose columns lie in $\ker(\boldsymbol{\alpha}^\top)$ is valid. For linear equality constraint~\eqref{eq:gp_constraint}, $\ker(\boldsymbol{\alpha}^\top)$ has dimension $Q - 1$. Since we assumed $\alpha_j \neq 0$, the set of admissible matrices is
\begin{equation}\label{eq:solution_set}
  \mathcal{S}_\ell = \Bigl\{\mathbf{C}_\mathcal{L} \in \mathbb{R}^{Q \times \ell} \;:\; C_{\mathcal{L},1,r} = -\frac{1}{\alpha_1}\sum_{j=2}^Q \alpha_j C_{\mathcal{L},j,r},\;\; r = 1, \ldots, \ell \Bigr\}.
\end{equation}
Each column of $\mathbf{C}_\mathcal{L}$ has $Q - 1$ free entries (the first one being determined by the constraint). Crucially, unlike deductive approaches, this construction is stable with respect to the choice of the constrained entry. Consequently, the non-uniqueness of $\mathbf{C}_\mathcal{L}$ introduces free parameters that are estimated jointly with the kernel hyperparameters of $\mathbf{g}$ by maximizing the marginal log-likelihood~\eqref{eq:mll}. This lets the optimization select the parametrization that best fits the data while guaranteeing constraint satisfaction by construction.

\paragraph{Choice of kernel for $\mathbf{g}$}
A natural choice is to model the $\ell$ components of $\mathbf{g}$ as independent
scalar GPs, each with its own kernel $k_r(\mathbf{x},\mathbf{x}')$, so that
$\mathbf{k}_g(\mathbf{x},\mathbf{x}') = \mathrm{diag}\big(k_1(\mathbf{x},\mathbf{x}'),\dots,k_\ell(\mathbf{x},\mathbf{x}')\big)$.
Substituting into~\eqref{eq:param_prior}, the induced kernel on $\mathbf{f}$ reads
\begin{equation}
  \mathbf{k}_f(\mathbf{x},\mathbf{x}')
  = \mathbf{C}_\mathcal{L}\,
    \mathrm{diag}\big(k_1(\mathbf{x},\mathbf{x}'),\dots,k_\ell(\mathbf{x},\mathbf{x}')\big)\,
    \mathbf{C}_\mathcal{L}^\top ,
\end{equation}
where the constraint matrix $\mathbf{C}_\mathcal{L}\in\mathbb{R}^{Q\times\ell}$
satisfies $\mathbf{C}_\mathcal{L}\in\mathcal{S}_\ell$, as defined in~\eqref{eq:solution_set}.
Instead of placing independent priors on the components of $\mathbf{g}$, here we consider a linear model of coregionalization, to allow a richer prior structure on $\mathbf{f}$. Following~\eqref{eq:lcm},
we set
\begin{equation}
  \mathbf{k}_g(\mathbf{x},\mathbf{x}')
  = \sum_{r=1}^{R} k_r(\mathbf{x},\mathbf{x}')\,\mathbf{B}_r ,
\end{equation}
where each $k_r$ is a scalar kernel and each $\mathbf{B}_r\in\mathbb{R}^{\ell\times\ell}$
is a positive semi-definite coregionalization matrix. Substituting this LCM prior
into~\eqref{eq:param_prior} yields the constrained kernel on $\mathbf{f}$
\begin{equation} \label{eq:constrained_lcm}
  \mathbf{k}_f(\mathbf{x},\mathbf{x}')
  = \mathbf{C}_\mathcal{L}
    \Big(\textstyle\sum_{r=1}^{R} k_r(\mathbf{x},\mathbf{x}')\,\mathbf{B}_r\Big)
    \mathbf{C}_\mathcal{L}^\top
  = \sum_{r=1}^{R} k_r(\mathbf{x},\mathbf{x}')\,
    \mathbf{C}_\mathcal{L}\mathbf{B}_r\mathbf{C}_\mathcal{L}^\top .
\end{equation}
The parametrization thus acts solely on the coregionalization matrices of the
LCM, replacing each $\mathbf{B}_r\in\mathbb{R}^{\ell\times\ell}$ by
$\mathbf{C}_\mathcal{L}\mathbf{B}_r\mathbf{C}_\mathcal{L}^\top\in\mathbb{R}^{Q\times Q}$,
while the scalar kernels $k_r$ remain free. Each
$\mathbf{C}_\mathcal{L}\mathbf{B}_r\mathbf{C}_\mathcal{L}^\top$ is positive
semi-definite and its columns lie in $\ker(\boldsymbol{\alpha}^\top)$, so the
constraint is encoded entirely through the coregionalization structure. For estimation by maximum likelihood~\eqref{eq:mll}, in practice we parametrize each $\mathbf{C}_\mathcal{L}\mathbf{B}_r\mathbf{C}_\mathcal{L}^\top= \widetilde{\mathbf{W}}_r\widetilde{\mathbf{W}}_r^\top \succeq 0 $ with a single matrix $\widetilde{\mathbf{W}}_r \in \mathcal{S}_\ell$, where each column of $\widetilde{\mathbf{W}}_r$ is constrained to be in $\ker(\boldsymbol{\alpha}^\top)$.



\paragraph{Illustrative example}
To illustrate the construction, consider the case $Q = 2$ with the constraint $\alpha_1 f_1(\mathbf{x}) + \alpha_2 f_2(\mathbf{x}) = 0$ and $\ell = 1$. So $\mathbf{C}_\mathcal{L} \in \mathbb{R}^{2 \times 1}$ and takes the form $\mathbf{C}_\mathcal{L} = [-\alpha_2/\alpha_1,\; 1]^\top$. Placing a scalar GP prior $g \sim \mathcal{GP}(\mu_g, k_g)$, the induced kernel on $\mathbf{f}$ is
\begin{equation}\label{eq:example_kernel}
  \mathbf{k}_f(\mathbf{x}, \mathbf{x}') = \mathbf{C}_\mathcal{L}\mathbf{C}_\mathcal{L}^\top k_g(\mathbf{x}, \mathbf{x}') = \begin{bmatrix} \alpha_2^2/\alpha_1^2 & -\alpha_2/\alpha_1 \\[2pt] -\alpha_2/\alpha_1 & 1 \end{bmatrix} k_g(\mathbf{x}, \mathbf{x}').
\end{equation}
For instance, with $\alpha_1 = 1,\, \alpha_2 = 1$ (i.e., $f_1 + f_2 = 0$), one recovers $\mathbf{k}_f(\mathbf{x}, \mathbf{x}') = \bigl[\begin{smallmatrix} 1 & -1 \\ -1 & 1 \end{smallmatrix}\bigr] k_g(\mathbf{x}, \mathbf{x}')$. Any prior or posterior sample from this model satisfies $f_1(\mathbf{x}) + f_2(\mathbf{x}) = 0$ everywhere. In comparison, an unconstrained ICM model where $\mathbf{k}_f(\mathbf{x}, \mathbf{x}') = \mathbf{B} \,k(\mathbf{x}, \mathbf{x}') $ would estimate the coregionalization matrix $\mathbf{B}$ from data. While the estimated matrix may converge to the same structure after optimization, it can never guarantee exact constraint satisfaction, even when the training data are noise-free.
\section{Linear equality constraints and deductive approaches}\label{sec:constraint}
As discussed in the introduction, a natural strategy for enforcing a linear equality constraint $\sum_{j=1}^Q \alpha_j f_j(\mathbf{x}) = 0$ consists in selecting one output index $l \in \{1, \ldots, Q\}$ and deducing its prediction from the others via the constraint relation:
\begin{equation}\label{eq:deduction}
  \hat{f}_{l}(\mathbf{x}) = -\frac{1}{\alpha_{l}} \sum_{j \neq l} \alpha_j \hat{f}_j(\mathbf{x}),
\end{equation}
where $\hat{f}_j(\mathbf{x})$ denotes the prediction of the $j$-th output by an unconstrained surrogate model trained on the $Q - 1$ remaining outputs. This deductive approach can be combined with any standard multi-output regression method, including independent GPs or LCM-based MOGPs, without any modification to their training procedure. Constraint satisfaction on the predictive mean is guaranteed by construction through~\eqref{eq:deduction}. Moreover, posterior samples can also be made constraint-compliant: ones sample jointly from the posterior of the $Q - 1$ modeled outputs and recovers the deduced output via~\eqref{eq:deduction} applied to each sample.
However, this strategy introduces an arbitrary choice across outputs: the deduced output $ l$ is not modeled from data but recovered algebraically, which means its prediction error accumulates the individual errors of all modeled outputs. The choice of $ l$ is a priori arbitrary, and different choices may lead to substantially different predictive performances. In this section, we investigate these effects empirically through a  benchmark that compares the deductive approach, under all possible deduction scenarios, against the constrained MOGP (CMOGP) presented in Section~\ref{sec:constrained_gp}, which models all outputs jointly and encodes the constraint in its covariance structure.

\subsection{Benchmark setup}\label{sec:deductive_setup} 
We consider a regression problem with $Q = 3$ scalar outputs subject to the constraint $f_1(\mathbf{x}) + f_2(\mathbf{x}) + f_3(\mathbf{x}) = 0$. The three underlying functions are defined on $\mathcal{X} = [0, 1]^3$ as follows: $f_1$ is the Ishigami function \citep{Ishigami1990}, $f_2$ is the Branin function~\citep{jamil2013literature} (rescaled to $[0,1]^3$), and $f_3$ is determined by the constraint $f_3 = -(f_1 + f_2)$. This choice is deliberate: $f_1$ and $f_2$ have markedly different complexity and smoothness properties, so that the difficulty of the regression task is heterogeneous across outputs. The constraint output $f_3$ inherits the combined complexity of both.
Three methods are compared:
\begin{itemize}
  \item CMOGP: the constrained multi-output GP of Section~\ref{sec:constrained_gp}, 
which models all three outputs jointly using a constrained  kernel~\eqref{eq:constrained_lcm} with $\g$  modeled by an LCM ($R=3$, $\ell=1$) and  
anisotropic Mat\'{e}rn $5/2$ covariance functions used for each scalar kernel $k_r$ ($r\in{1,2,3}$). 
  \item Indep.\ GP: independent single-output GPs, one per modeled output, each with an anisotropic Mat\'{e}rn $5/2$ kernel, combined with deduction~\eqref{eq:deduction}.
  \item LCM: a multi-output GP with an LCM kernel ($R = 2$, rank-1 coregionalization matrices, anisotropic Mat\'{e}rn $5/2$ scalar kernels), trained on the $Q - 1 = 2$ modeled outputs, combined with deduction~\eqref{eq:deduction}.
\end{itemize}

We consider three deduction scenarios $ l \in \{1, 2, 3\}$, for every deductive baseline. A key practical distinction is that  CMOGP requires only a single training phase to predict all outputs jointly (each deductive configuration mandates a separately trained model). Our evaluation relies on 200 independent replications. Within a single replication, we sample a fixed test set of 200 points alongside varying training sets of size $N \in \{20, 50, 100\}$ using Latin Hypercube Sampling (LHS). Model hyperparameters are optimized by maximizing the marginal log-likelihood (via L-BFGS-B), using a multi-start strategy with 50 random initializations for multi-output formulations (CMOGP and LCM) and 30 for standard single-output GPs.
Predictive performance is evaluated by the  Root Mean Square Error (RMSE) on each output, defined for $N_{test}$ test samples as:
\begin{equation} \text{RMSE} = \sqrt{\frac{1}{n} \sum_{i=1}^{N_{test}} (y_i - \hat{y}_i)^2}. \end{equation}
Furthermore, the quality of the uncertainty estimates is evaluated through the average length of the $90\%$ predictive intervals. 
\subsection{Results}\label{sec:deductive_results}
We first present the RMSE results on the output $f_2$ (Branin) to
analyze the impact of the arbitrary choice of deduced output, and then turn to a win rate analysis covering all three outputs
 
\paragraph{RMSE across deduction scenarios}
Figure~\ref{fig:rmse_F2} presents the empirical distributions of the RMSE associated with the prediction of $f_2$, for training set sizes $N \in {20, 50, 100}$ and across the three deduction scenarios. Each panel corresponds to a specific choice of deduced output $l$, while the CMOGP model yields identical results across panels, as it is independent of this choice.
A first insight is obtained by contrasting the three panels. When $f_2$ is explicitly modeled (i.e., $l \in \{1,3\}$), the three approaches exhibit comparable performance, as reflected by largely overlapping RMSE distributions. In contrast, a markedly different behavior emerges in the intermediate panel ($l = 2$), where $f_2$ is reconstructed. In this setting, the independent GP displays a noticeable degradation in performance, characterized by both an upward shift in RMSE values. 
This discrepancy highlights a pronounced sensitivity of the independent GP to the choice of deduced output. By comparison, the LCM model appears more stable, with RMSE distributions that remain at a similar level regardless of whether $f_2$ is directly modeled or inferred. While these distributions suggest comparable behavior between the LCM and the CMOGP, such a visual assessment remains limited, as it does not capture how often one method effectively outperforms the others. To address this point, we now turn to a complementary win rate analysis.

\begin{figure}
  \centering
  \includegraphics[width=\textwidth]{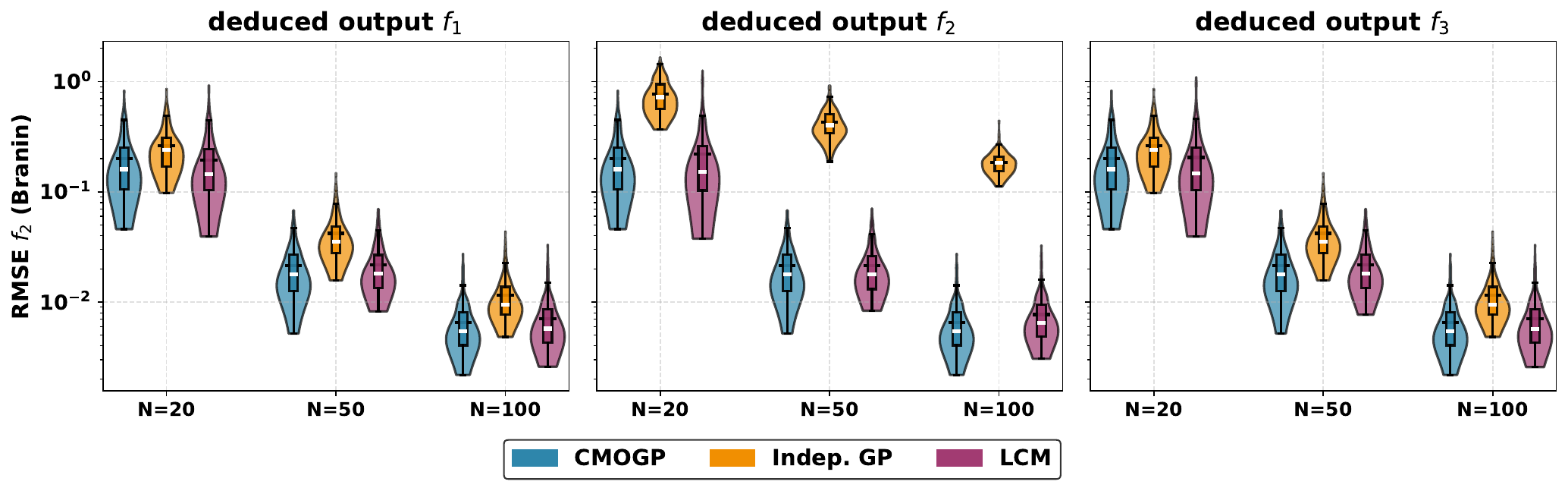}
  \caption{ Distribution of the RMSE on $f_2$ (Branin) across $200$ replications, for training set sizes $N \in \{20, 50, 100\}$ and the three deduction scenarios. Each panel corresponds to a different choice of deduced output $ l$. The CMOGP (blue) is identical across panels. The deductive methods exhibit sensitivity to the choice of $ l$, with the independent GP showing systematic degradation when $f_2$ is deduced (middle panel).}
  \label{fig:rmse_F2}
\end{figure}
 
\paragraph{Win rate analysis}
Figure~\ref{fig:winrate_n20} reports the win rate of each method for $N = 20$, defined as the proportion of replications in which a given method achieves the lowest RMSE for a given output. The three panels correspond to performance evaluated on $f_1$, $f_2$, and $f_3$, respectively, while the horizontal axis indicates the deduced output $l$. 
This view complement Figure~\ref{fig:rmse_F2} in two ways. First, for each evaluated output, the win rates results change across deduction scenarios, with a noticeable shifts when the evaluated quantity is deduced. Since the CMOGP approach yields identical RMSE distributions regardless of the deduction choice, the observed variations in win rates necessarily reflect changes in the relative performance of the two deductive methods (independent GP and LCM). Figure~\ref{fig:winrate_n20} does not allow one to attribute this variability to one deductive method alone. However, the result is consistent with the
heterogeneity already observed on the independent GP in
Figure~\ref{fig:rmse_F2}.
Second, Figure~\ref{fig:winrate_n20} separates the LCM from the CMOGP in a way that the violin plot did not allow. Across outputs and deduction scenarios, the CMOGP wins more often than the LCM in most configurations, even when their RMSE distributions looked close in Figure~\ref{fig:rmse_F2}. The gap in favor of the CMOGP grows with $N$. The corresponding plots for $N = 50$ and $N = 100$ are reported in~\ref{app:deductive}.
\begin{figure}
  \centering
  \includegraphics[width=\textwidth]{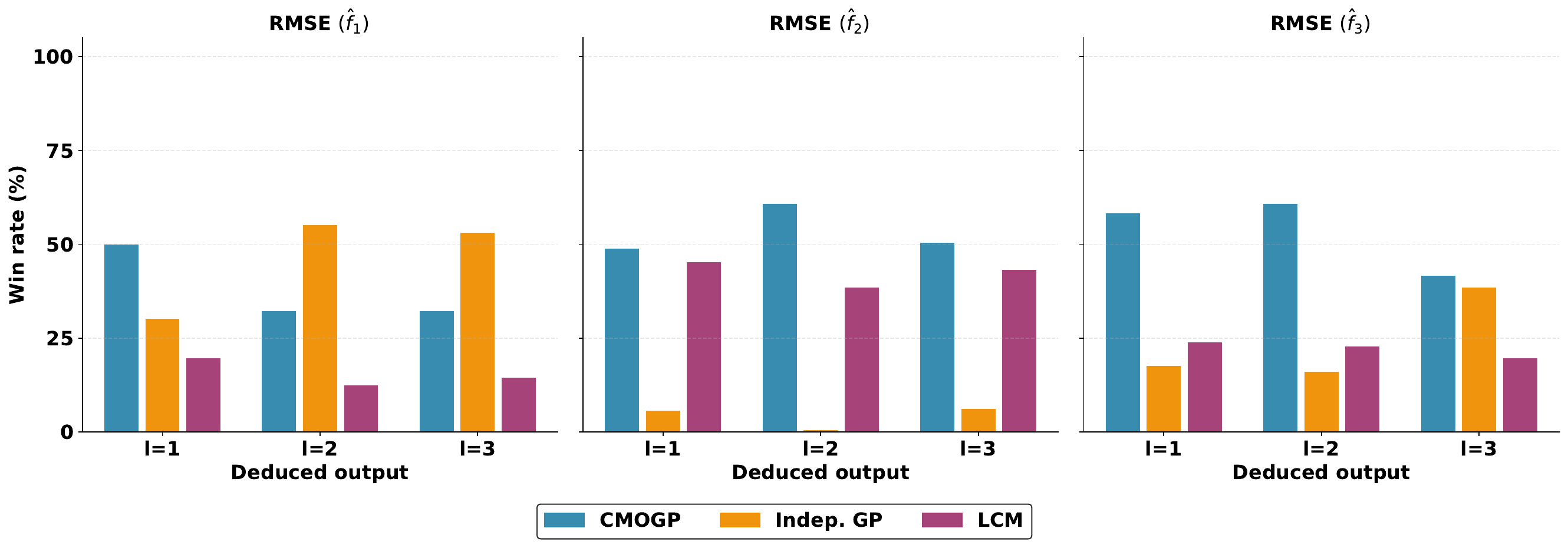}
\caption{ Win rate (\%) of each method at $N = 20$ on the RMSE of each output ($f_1$, $f_2$, $f_3$), across the three deduction scenarios ($l \in \{1, 2, 3\}$). The CMOGP (blue) is the same model in all scenarios. Aggregated over 200 replications.}
\label{fig:winrate_n20}
\end{figure}
 
\subsection{Discussion}\label{sec:deductive_discussion}
 Two main observations emerge from these results.
First, the deductive approach can be sensitive to the choice of which output is deduced. This sensitivity is directly visible on the RMSE of the independent GP, which shifts towards higher values when the evaluated output is the deduced one (Figure~\ref{fig:rmse_F2}, middle panel), and this gap does not vanish with the training set size. The win rate analysis provides an additional piece of evidence: since the CMOGP behaves identically across scenarios, the variability of the win rate ranking between scenarios necessarily comes from the deductive methods. In practice, this means that with a deductive strategy, the choice of the deduced output cannot be made without trying all $Q$ configurations, which multiplies the cost by $Q$ and weakens the scalability of the approach.

Second, between the two deductive methods, the LCM looks more stable than the independent GP on the RMSE distributions of
Figure~\ref{fig:rmse_F2}: modeling the $Q-1$ retained outputs jointly seems to partly absorb the effect of the deductive choice.
Among the joint approaches in the benchmark, the win rate analysis shows that the CMOGP is the best choice. The LCM approach still treats one output as a quantity computed from the others, so the constraint information is only used at the post-processing step. The CMOGP, on the other hand, models all
$Q$ outputs together and encodes the constraint analytically in its kernel. This combination of joint modeling of all $Q$ outputs and analytical use of the constraint translates into the win rate plots, where the CMOGP wins more often than the LCM in most configurations, with a gap that grows with
$N$. We also examined the lengths of the predictive intervals, after verifying that all three methods achieve comparable empirical coverage rates close to the nominal level. The patterns observed on the interval length distributions are similar to those of the RMSE in Figure~\ref{fig:rmse_F2} and do not bring additional information beyond what is already discussed above, see Figure~\ref{fig:interval_len_F2} in~\ref{app:interval_len}.

These findings clearly demonstrate the inherent instability of deductive approaches, which illustrates the need to handle equality constraints through robust frameworks such as the CMOGP. While this motivates to adopt the CMOGP framework for our current problem, a major challenge arises regarding the dimensionality of the outputs.
Since constrained MOGP regression cannot be applied directly when the output dimension $P = Q \times S$ is large, a dimensionality reduction step that preserves the constraint structure in the latent space is needed. This is the subject of the next section.

\section{Principal component analysis for multi-field data}\label{sec:pca_theory}
In this section, we address the problem of reducing the dimension of the output fields while preserving the linear constraint structure, so that the constrained MOGP of Section~\ref{sec:constrained_gp} can be applied in the resulting latent space. We present three PCA strategies for multi-field data, analyze their relationship with the constraint, and establish theoretical bounds comparing their reconstruction quality.

\subsection{PCA strategies for multi-field data}\label{sec:pca_strategies}

For each field $k = 1, \ldots, Q$, we collect the training observations into the data matrix $\mathbf{Y}_k \in \mathbb{R}^{N \times S}$ whose $i$-th row is ${\mathbf{y}^{(i)}_k}^\top$, where $\mathbf{y}^{(i)}_k$ is the $k$-th component of the simulator output introduced in Section~\ref{sec:formulation}. We assume that the sample mean has been subtracted from each row of these data matrices. The non-centered case will be dealt with in Section~\ref{sec:centering}. The singular value decomposition (SVD) of each field is $\mathbf{Y}_k = \mathbf{U}_k \mathbf{D}_k \mathbf{V}_k^\top$, where $\mathbf{V}_k \in \mathbb{R}^{S \times S}$ contains the right singular vectors (spatial basis), $\mathbf{U}_k \in \mathbb{R}^{N \times N}$ the left singular vectors (observation basis), and $\mathbf{D}_k$ the singular values in decreasing order. We define the spatial covariance matrix $\boldsymbol{\Sigma}_k = \frac{1}{N}\mathbf{Y}_k^\top \mathbf{Y}_k \in \mathbb{R}^{S \times S}$ and the Gram matrix $\mathbf{K}_k = \frac{1}{S}\mathbf{Y}_k \mathbf{Y}_k^\top \in \mathbb{R}^{N \times N}$. For a truncation rank $m$, we denote by $\mathbf{V}_{k,m}$ and $\mathbf{U}_{k,m}$ the matrices of the first $m$ right and left singular vectors, and by $\mathbf{P}_{k,m} = \mathbf{V}_{k,m}\mathbf{V}_{k,m}^\top$ and $\boldsymbol{\Pi}_{k,m} = \mathbf{U}_{k,m}\mathbf{U}_{k,m}^\top$ the corresponding spatial and observation projectors.
Three strategies can be considered for simultaneously reducing the dimension of all $Q$ fields.
\paragraph{Field-wise PCA} Each field $\mathbf{Y}_k$ is reduced independently onto its own optimal rank-$m$ spatial basis $\mathbf{V}_{k,m}$. The reconstruction error for field $k$ is
\begin{equation}\label{eq:E_Field}
  E_{k,m}^{(\text{field})} = \|\mathbf{Y}_k(\mathbf{I}_S - \mathbf{P}_{k,m})\|_F^2.
\end{equation}
This approach minimizes the reconstruction error for each field individually, but uses a different basis per field, which prevents any coupling between the latent representations.

\paragraph{Column-wise PCA} The fields are concatenated column-wise as $\mathbf{Y}_{\text{col}} = [\mathbf{Y}_1, \ldots, \mathbf{Y}_Q] \in \mathbb{R}^{N \times QS}$, and a single PCA is performed on $\mathbf{Y}_{\text{col}}$. The projector operates in the observation space: $\boldsymbol{\Pi}_{\text{col},m} = \mathbf{U}_{\text{col},m}\mathbf{U}_{\text{col},m}^\top$, where $\mathbf{U}_{\text{col},m} $ contains the left singular vectors associated with the first $m$ eigenvalues of the SVD decomposition of $\mathbf{Y}_{\text{col}}$. The global and per-field reconstruction errors are
\begin{equation}\label{eq:E_col}
  E_m^{(\text{col})} = \sum_{k=1}^Q E_{k,m}^{(\text{col})}, \qquad E_{k,m}^{(\text{col})} = \|(\mathbf{I}_N - \boldsymbol{\Pi}_{\text{col},m})\mathbf{Y}_k\|_F^2.
\end{equation}
This approach captures inter-field correlations at the dimensionality reduction stage and produces scalar latent variables that can later be modeled by independent GPs.

\paragraph{Row-wise PCA} The fields are stacked row-wise as $\mathbf{Y}_{\text{row}} = [\mathbf{Y}_1^\top, \ldots, \mathbf{Y}_Q^\top]^\top \in \mathbb{R}^{NQ \times S}$. Through the SVD of $\mathbf{Y}_{row}$, we obtain a shared spatial basis $\mathbf{V}_{\text{row},m}$  formed by its first $m$ leading right singular vectors, along with the associated projector $\mathbf{P}_{\text{row},m} = \mathbf{V}_{\text{row},m}\mathbf{V}_{\text{row},m}^\top$. The errors are
\begin{equation}\label{eq:E_row}
  E_m^{(\text{row})} = \sum_{k=1}^Q E_{k,m}^{(\text{row})}, \qquad E_{k,m}^{(\text{row})} = \|\mathbf{Y}_k(\mathbf{I}_S - \mathbf{P}_{\text{row},m})\|_F^2.
\end{equation}
Because all fields share the same spatial basis, the projection of each sample $\mathbf{y}^{(i)} = [\mathbf{y}^{(i)\top}_1, \ldots, \mathbf{y}^{(i)\top}_Q]^\top$ onto $\mathbf{V}_{\text{row},m}$ produces a vector-valued weight $\mathbf{w}_i \in \mathbb{R}^{Q \times m}$, whose components of each column can be modeled jointly by a MOGP.

A key structural property distinguishes the row-wise approach. The row-wise covariance matrix satisfies $\boldsymbol{\Sigma}_{\text{row}} = \frac{1}{Q}\sum_{k=1}^Q \boldsymbol{\Sigma}_k$, and similarly $\mathbf{K}_{\text{col}} = \frac{1}{Q}\sum_{k=1}^Q \mathbf{K}_k$ for the column-wise Gram matrix.

\subsection{Constraint preservation in the latent space}\label{sec:pca_constraint}

Consider the constraint $\sum_{j=1}^Q \alpha_j \mathbf{Y}_j = \mathbf{0}_{N \times S}$ satisfied by the training data. Forming the row-wise concatenation and denoting by $\mathbf{U} = [\alpha_1 \mathbf{I}_N, \ldots, \alpha_Q \mathbf{I}_N] \in \mathbb{R}^{N \times NQ}$ the constraint matrix, the constraint reads $\mathbf{U}\mathbf{Y}_{\text{row}} = \mathbf{0}$. Since the row-wise projector $\mathbf{P}_{\text{row},m}$ acts on the right (spatial dimension), we have:
\begin{equation}\label{eq:constraint_pca}
  \mathbf{U}\mathbf{Y}_{\text{row}} = \mathbf{0} \implies \mathbf{U}\mathbf{W} = \mathbf{U}\mathbf{Y}_{\text{row}}\mathbf{V}_{\text{row},m} = \mathbf{0},
\end{equation}
where $\mathbf{W} \in \mathbb{R}^{NQ \times m}$ denotes the latent weight matrix. The constraint is thus exactly transferred to the latent space, {regardless of the truncation rank $m$}. Furthermore, the reconstruction $\tilde{\mathbf{Y}}_{\text{row}} = \mathbf{W}\mathbf{V}_{\text{row},m}^\top$ also satisfies the constraint:
\begin{equation}\label{eq:constraint_recon}
  \mathbf{U}\tilde{\mathbf{Y}}_{\text{row}} = \mathbf{U}\mathbf{W}\mathbf{V}_{\text{row},m}^\top = \mathbf{0}.
\end{equation}
This property is specific to the row-wise approach. For the column-wise strategy, the projector operates on the observation dimension, and no analogous preservation holds in general. For the field-wise approach, the use of different spatial bases per field destroys any linear relation between the latent representations. Consequently, the row-wise PCA is the only strategy among the three that can be coupled with the constrained MOGP of Section~\ref{sec:constrained_gp} to guarantee constraint satisfaction throughout the entire algorithm.

\subsection{Theoretical analysis of the reconstruction error of multi-field approaches}\label{sec:bounds}

Since the row-wise approach is the only one that preserves the constraint, while the column-wise and field-wise approaches can only be used within a deductive framework, a natural question is whether they can also be compared in terms of reconstruction error ? This allows us to assess whether adopting the row-wise approach requires a significant trade-off in accuracy.
To evaluate this, we compare the row-wise and col-wise approaches (referred to as multi-field approaches) with the field-wise optimum, which by construction minimizes the reconstruction error for each field individually. The proofs of all results in this section are given in~\ref{app:proofs}.

\begin{theorem}[Row-wise excess error bound]\label{thm:row}
Let $\delta_k = \lambda_m^{(k)} - \lambda_{m+1}^{(k)}$ denote the spectral gap of $\boldsymbol{\Sigma}_k$ at rank $m$. If $\delta_k > 0$, then:
\begin{equation}\label{eq:bound_row}
  \Delta E_{k,m}^{(\text{row})} := E_{k,m}^{(\text{row})} - E_{k,m}^{(\text{field})} \;\leq\; \frac{2N\sqrt{2}\,\|\boldsymbol{\Sigma}_k\|_F}{\delta_k} \cdot \min\bigl\{\sqrt{m}\,\|\boldsymbol{\Delta\Sigma}_k\|_{\text{op}},\;\|\boldsymbol{\Delta\Sigma}_k\|_F\bigr\},
\end{equation}
where $\boldsymbol{\Delta\Sigma}_k = \boldsymbol{\Sigma}_{\text{row}} - \boldsymbol{\Sigma}_k = \frac{1}{Q}\sum_{j=1}^Q (\boldsymbol{\Sigma}_j - \boldsymbol{\Sigma}_k)$ and $\|\cdot\|_{\text{op}}$ denotes the spectral norm.
\end{theorem}
 
\begin{theorem}[Column-wise excess error bound]\label{thm:col}
Let $\delta_k^{(\text{col})} = \lambda_m^{(k)} - \lambda_{m+1}^{(k)}$ denote the spectral gap of $\mathbf{K}_k$ at rank $m$. If $\delta_k^{(\text{col})} > 0$, then:
\begin{equation}\label{eq:bound_col}
  \Delta E_{k,m}^{(\text{col})} := E_{k,m}^{(\text{col})} - E_{k,m}^{(\text{field})} \;\leq\; \frac{2S\sqrt{2}\,\|\mathbf{K}_k\|_F}{\delta_k^{(\text{col})}} \cdot \min\bigl\{\sqrt{m}\,\|\boldsymbol{\Delta K}_k\|_{\text{op}},\;\|\boldsymbol{\Delta K}_k\|_F\bigr\},
\end{equation}
where $\boldsymbol{\Delta K}_k = \mathbf{K}_{\text{col}} - \mathbf{K}_k = \frac{1}{Q}\sum_{j=1}^Q (\mathbf{K}_j - \mathbf{K}_k)$.
\end{theorem}
 
The two bounds share a common structure: the excess error for field $k$ is controlled by the inter-field spectral heterogeneity, the distance between the covariance (or Gram) matrix of field $k$ and the average over all fields. When all fields share similar structures, both bounds tend to zero and the multi-field approaches incur essentially no penalty relative to the field-wise optimum. The two bounds involve different matrices operating in different spaces ($\mathbb{R}^{S \times S}$ vs.\ $\mathbb{R}^{N \times N}$), with different spectral gaps, so their relative ordering is problem-dependent.
 
\begin{corollary}[Vanishing excess error under shared basis]\label{cor:zero}
If all fields share the same spatial basis ($\mathbf{V}_1 = \cdots = \mathbf{V}_Q$), then $\Delta E_{k,m}^{(\text{row})} = 0$ for all $k$. Similarly, if all fields share the same observation basis ($\mathbf{U}_1 = \cdots = \mathbf{U}_Q$), then $\Delta E_{k,m}^{(\text{col})} = 0$ for all $k$.
\end{corollary}

\subsection{Direct comparison of row-wise and column-wise}\label{sec:comparison}
 
The previous bounds compare each multi-field approach to a separate reference (the field-wise optimum). Theorem~\ref{thm:dominance} below provides an exact criterion to compare row-wise and column-wise PCA against each other,
when both are applied to the full set of $Q$ fields.
 
\begin{theorem}[Row-wise vs.\ column-wise global error]\label{thm:dominance}
The difference between the global reconstruction errors of the column-wise and row-wise approaches at rank $m$ is exactly:
\begin{equation}\label{eq:dominance}
  E_m^{(\text{col})} - E_m^{(\text{row})} = \text{Tr}(\mathbf{D}_{\text{row},m}^2) - \text{Tr}(\mathbf{D}_{\text{col},m}^2),
\end{equation}
where $\mathbf{D}_{\text{row},m}^2$ and $\mathbf{D}_{\text{col},m}^2$ are the diagonal matrices of the $m$ largest squared singular values of $\mathbf{Y}_{\text{row}}$ and $\mathbf{Y}_{\text{col}}$, respectively. Equivalently, denoting by $\lambda_1^{(\text{row})} \geq \cdots \geq \lambda_m^{(\text{row})}$ and $\lambda_1^{(\text{col})} \geq \cdots \geq \lambda_m^{(\text{col})}$ the $m$ largest eigenvalues of $\boldsymbol{\Sigma}_{\text{row}} = \frac{1}{Q}\sum_k \boldsymbol{\Sigma}_k$ and $\mathbf{K}_{\text{col}} = \frac{1}{Q}\sum_k \mathbf{K}_k$:
\begin{equation}\label{eq:dominance_eig}
  E_m^{(\text{col})} - E_m^{(\text{row})} = NQ \sum_{p=1}^m \lambda_p^{(\text{row})} - QS \sum_{p=1}^m \lambda_p^{(\text{col})}.
\end{equation}
In particular, neither approach dominates the other in general.
\end{theorem}
 
\subsection{Discussion}\label{sec:bounds_discussion}
The theoretical results of this section provide guarantees for the row-wise PCA strategy that underlies our framework.
Theorems~\ref{thm:row} and~\ref{thm:col} identify inter-field spectral heterogeneity as the sole quantity governing the excess reconstruction error of each multi-field approach relative to the field-wise optimum. For problems where the component fields share similar spatial covariance structures, which is common in physical applications where the fields arise from the same PDE system and are defined on the same mesh, the row-wise approach incurs a negligible penalty while offering the crucial advantage of preserving the linear equality constraint.
Theorem~\ref{thm:dominance} gives an exact criterion to compare row-wise and column-wise PCA on the full set of $Q$ fields. We note that, in the constrained setting, column-wise PCA is in fact applied on $Q-1$ fields rather than $Q$, since one field is removed and recovered by deduction. By optimality of PCA on the larger set, this can only increase the
reconstruction error compared to the column-wise strategy on $Q$ fields. As a consequence, whenever the criterion of Theorem~\ref{thm:dominance} indicates equivalence or row-wise dominance on the $Q$ fields, the row-wise approach also dominates the deductive column-wise variant.

From a practical standpoint, the bounds of Theorems~\ref{thm:row} and~\ref{thm:col} and the criterion of Theorem~\ref{thm:dominance} can be evaluated directly from the data. They give a way of assessing which
regime one is in: similar fields, heterogeneous fields, row wise or column-wise dominance. This is useful in practice when no prior information is available on the choice between the strategies, and even more so when one suspects the fields to be correlated.

\subsection{The \texttt{Row-CMO} framework}\label{sec:pcacgp}

We now combine the row-wise PCA with the constrained MOGP of Section~\ref{sec:constrained_gp} to form the \texttt{Row-CMO} framework. Recall that the constraint under consideration is $\sum_{j=1}^Q \alpha_j \mathbf{f}_j(\mathbf{x}) = \mathbf{0}_S$, where the coefficients $\alpha_j$ are non-zero scalars independent of $\mathbf{x}$. The complete procedure is summarized in Algorithm~\ref{alg:pcacgp}.

\begin{algorithm}[t]
\caption{\texttt{Row-CMO}}\label{alg:pcacgp}
\begin{algorithmic}[1]
\REQUIRE Training data $\bigl((\mathbf{x}^{(i)})_{i=1}^N,\, \mathbf{Y}_1, \ldots, \mathbf{Y}_Q\bigr)$, coefficients $\boldsymbol{\alpha} = (\alpha_1, \ldots, \alpha_Q)$, and rank $m$.
\ENSURE Constrained predictive model $\hat{\mathbf{f}}$ satisfying $\sum_{j=1}^Q \alpha_j \hat{\mathbf{f}}_j = \mathbf{0}_S$.
\STATEx \hspace*{-2.em} \textbf{Step 1: Data Pre-processing (Training)}
\STATE Center each field: $\mathbf{Y}_j \leftarrow \mathbf{Y}_j - \mathbf{1}_N \bar{\mathbf{y}}^{\top}_j$, where $\bar{\mathbf{y}}_j$ is the empirical mean of field $j$.
\STATE Stack the centered fields row-wise into $\mathbf{Y}_{\text{row}} \in \mathbb{R}^{NQ \times S}$ and extract the $m$-rank spatial basis $\mathbf{V}_{\text{row},m} \in \mathbb{R}^{S \times m}$ via SVD.
\STATE Project data into the latent space: $\mathbf{w}_j^{(i)} = \mathbf{y}^{(i)\top}_j\, \mathbf{V}_{\text{row},m} \in \mathbb{R}^m$ 
\STATEx \hspace*{-2.em} \textbf{Step 2: Model Training}
\FOR{each latent dimension $s = 1, \ldots, m$}
    \STATE Fit a constrained MOGP on $\bigl((\mathbf{x}^{(i)})_{i=1}^N,\, (w^{(i)}_{1,s}, \ldots, w^{(i)}_{Q,s})_{i=1}^N\bigr)$ using the kernel of Eq.~\eqref{eq:param_prior}.
\ENDFOR
\STATEx \hspace*{-2.em} \textbf{Step 3: Prediction at a new input $\mathbf{x}_*$}
\FOR{each latent dimension $s = 1, \ldots, m$}
    \STATE Compute the latent prediction $\hat{\mathbf{w}}_{*,s} = [\hat{w}_{*,1,s}, \ldots, \hat{w}_{*,Q,s}]^\top$ from the $s$-th CMOGP.
\ENDFOR
\STATE Assemble per-field latent predictions: $\hat{\mathbf{w}}_{*,j} = [\hat{w}_{*,j,1}, \ldots, \hat{w}_{*,j,m}]^\top$ for $j = 1, \ldots, Q$.
\STATE Reconstruct each physical field $\hat{\mathbf{y}}_{*,j} = \bar{\mathbf{y}}_j + \mathbf{V}_{\text{row},m}\, \hat{\mathbf{w}}_{*,j}$ for $j = 1, \ldots, Q$.
\end{algorithmic}
\end{algorithm}
By construction, the constraint is satisfied at every stage of Algorithm~\ref{alg:pcacgp}: in the latent weights (line~5), in the MOGP predictions and posterior samples (line~8), and in the reconstructed fields (line~11). The resulting predictive mean and any posterior sample thus satisfy $\sum_j \alpha_j \hat{\mathbf{f}}_j(\mathbf{x}) = \mathbf{0}_S$ to machine precision, for all $\mathbf{x} \in \mathcal{X}$.

 \subsubsection{Extension to noisy data}
\label{sec:noisy_data}

The framework presented above assumes noise-free data, which is the typical situation for outputs of deterministic numerical simulators. In applications where the outputs fields are corrupted by observation noise, two steps of the procedure require an adjustment: the centering step in
Algorithm~\ref{alg:pcacgp}, and the GP regression in the latent space.
\paragraph{Constraint-consistent centering}\label{sec:centering}
In the noise-free case, the field means $\bar{\mathbf{y}}_k$ naturally satisfy $\sum_j \alpha_j \bar{\mathbf{y}}^j = \mathbf{0}_S$, so standard centering preserves the constraint on the centered data. When the outputs data are corrupted by noise, however, the empirical means may not satisfy the constraint exactly: $\sum_j \alpha_j \bar{\mathbf{y}}^j = \mathbf{r} \neq \mathbf{0}_S$, where $\mathbf{r} \in \mathbb{R}^S$ is a residual vector. If left uncorrected, this residual propagates into the centered data and breaks
the constraint structure that is essential for the row-wise PCA to transfer the constraint to the latent space.

To address this, we use a constraint-consistent centering that
distributes the residual proportionally across the field means.
Specifically, we define adjusted means as
\begin{equation}
\tilde{\bar{\mathbf{y}}}_k = \bar{\mathbf{y}}_k -
\frac{\alpha_k}{\sum_{j=1}^Q (\alpha_j)^2}\, \mathbf{r},
\qquad k = 1, \ldots, Q,
\label{eqn:adjusted_means}
\end{equation}
which satisfy $\sum_j \alpha_j \tilde{\bar{\mathbf{y}}}^j = \mathbf{0}_S$
by construction. This adjustment is the minimum-norm correction (in the
Euclidean sense on the vector $[\bar{\mathbf{y}}^1, \ldots,
\bar{\mathbf{y}}^Q]$) that restores the constraint on the means. The
centering of Algorithm~\ref{alg:pcacgp} (line 1) is then performed using
$\tilde{\bar{\mathbf{y}}}_k$ in place of $\bar{\mathbf{y}}_k$.

\paragraph{Constrained MOGP regression with noise}
The constrained MOGP fitted on the latent components in
Step~2 of Algorithm~\ref{alg:pcacgp} can be extended to handle
noisy observations through the standard MOGP likelihood. Assuming an observation model $\mathbf{w}_i^s = \mathbf{h}_s(\mathbf{x}^{(i)}) +
\boldsymbol{\eta}_i^s$ for each latent dimension $s$, with independent Gaussian noise $\boldsymbol{\eta}_i^s \sim \mathcal{N}(\mathbf{0}, \boldsymbol{\Sigma}^s)$ and
$\boldsymbol{\Sigma}^s = \mathrm{diag}(\sigma_{s,1}^2, \ldots,
\sigma_{s,Q}^2)$ collecting the per-output noise variances, the covariance matrix to invert in the posterior expressions becomes $\mathbf{k}(\mathbf{X}, \mathbf{X}) + \boldsymbol{\Sigma}^s \otimes
\mathbf{I}_N$. The noise variances are estimated jointly with the kernel hyperparameters by maximizing the marginal log-likelihood. The rest of the procedure (Steps 1--3 of Algorithm~\ref{alg:pcacgp}) is unchanged.

\subsubsection{Handling more general constraints}\label{sec:generalization}
 
Algorithm~\ref{alg:pcacgp} is formulated for the constraint $\sum_j \alpha_j \mathbf{f}_j(\mathbf{x}) = \mathbf{0}_S$ with scalar coefficients $\alpha_j$ independent of $\mathbf{x}$. We now discuss two generalizations of practical interest.
 
\paragraph{Input-dependent coefficients} When the constraint coefficients depend on the input, $\sum_j \alpha_j(\mathbf{x})\, f_j(\mathbf{x}) = 0$, with each $\alpha_j(\cdot)$ known and non-vanishing on $\mathcal{X}$, one can introduce the transformed variables $z_i^j = \alpha_j(\mathbf{x}^{(i)})\, y_i^j$ for each observation $i$ and field $j$. These satisfy $\sum_j z_i^j = 0$, which is a constraint with unit scalar coefficients. Algorithm~\ref{alg:pcacgp} is then applied to the matrices $\mathbf{Z}_k$ with $\alpha_j = 1$ for all $j$. The inverse transformation $\hat{y}_*^j = \hat{z}_*^j / \alpha_j(\mathbf{x}_*)$ is well defined since the coefficients are non-zero by assumption. 
 
\paragraph{Non-zero second member} When the constraint takes the form $\sum_j \alpha_j \mathbf{f}_j(\mathbf{x}) = \mathbf{c}(\mathbf{x})$ with a known, non-zero function $\mathbf{c} : \mathcal{X} \to \mathbb{R}^S$, the goal is to reduce the problem to a homogeneous constraint by defining transformed variables that sum to zero. This reduction requires subtracting an appropriate share of $\mathbf{c}(\mathbf{x})$ from each output field. A natural approach is to define, for each field $k$:
\begin{equation}\label{eq:transform_c}
  \tilde{\mathbf{y}}^{(i)}_k = \mathbf{y}^{(i)}_k - \frac{\beta_k}{\sum_j \alpha_j \beta_j}\, \mathbf{c}(\mathbf{x}^{(i)}), \quad k = 1, \ldots, Q,
\end{equation}
where $\beta_k > 0$ are  weights satisfying $\sum_j \alpha_j \beta_j \neq 0$. The transformed variables satisfy $\sum_j \alpha_j \tilde{\mathbf{y}}_j^{(i)} = 0$ for any choice of $\boldsymbol{\beta}$, and Algorithm~\ref{alg:pcacgp} can be applied to $\{\tilde{\mathbf{Y}}_k\}$.
 
The choice of $\boldsymbol{\beta}$, however, is not neutral: subtracting a spatially complex function $\mathbf{c}(\mathbf{x})$ from one or several fields modifies their variance structure, which in turn affects the row-wise PCA reconstruction quality. In particular, assigning the entire correction to a single output (e.g., $\beta_l = 1/\alpha_l$ and $\beta_k = 0$ for $k \neq l$) would reintroduce an arbitrary analogous to the arbitrary output selection in the deductive approach. Uniform weights ($\beta_k = \alpha_k$, recovering~\eqref{eqn:adjusted_means}) avoid this issue but may not be optimal when the fields have heterogeneous variance levels. In practice, the weights can be guided by domain knowledge or by a variance-based criterion: a field $k$ for which the transformation $\mathbf{y}^{(i)}_k - \beta_k \mathbf{c}(\mathbf{x}^{(i)})$ does not significantly increase the empirical variance is a natural candidate for absorbing a larger share of $\mathbf{c}(\mathbf{x})$. A systematic study of optimal redistribution strategies in this setting is an interesting direction for future work.

 
\section{Numerical experiments}\label{sec:experiments}
 
We now validate the \texttt{Row-CMO} framework on two case studies of increasing complexity. The first is a population dynamics problem based on the Lotka--Volterra system, which serves as a controlled test bed for evaluating the theoretical results of Section~\ref{sec:pca_theory} and confirming the sensitivity of deductive approaches observed in Section~\ref{sec:constraint}. The second is a Computational Fluid Dynamics (CFD) application involving the prediction of Reynolds stress tensor components under an incompressibility constraint, which demonstrates the applicability of \texttt{Row-CMO} to a realistic high-dimensional industrial problem.
 
In both case studies, we compare \texttt{Row-CMO} against deductive alternatives that combine different PCA strategies with unconstrained GP regression. The competing methods are:
\begin{itemize}
  \item \texttt{Row-CMO} (our model): Row-wise PCA + Constrained MOGP;
  \item \texttt{Col-Indep$(l)$}: Column-wise PCA + Independent GPs with deduced output $l$;
  \item \texttt{Fw-Indep$(l)$}: Field-wise PCA + Independent GPs with deduced output $l$;
  \item \texttt{Fw-LCM$(l)$}: Field-wise PCA + LCM with deduced output $l$.
\end{itemize}
For each deductive method, $l$ ranges over all $Q$ possible deduction scenarios. The comparison focuses on predictive accuracy  as a function of the latent dimension $m$ and the training set size $N$.
The Relative Root Mean Square Error (RRMSE) provides a global measure of accuracy for each predicted field $\hat{\mathbf{y}}^{(i)}$ relative to the simulated reference $\mathbf{y}^{(i)}$:
\begin{equation}\label{eq:rrmse}
  \text{RRMSE}^2 = \frac{1}{N_{\text{test}}} \sum_{i=1}^{N_{\text{test}}} \frac{\|\hat{\mathbf{y}}^{(i)} - \mathbf{y}^{(i)}\|_2^2}{S\,\|\mathbf{y}^{(i)}\|_\infty^2}
\end{equation}
on a test set of size $N_{test}$.
Hyperparameters are optimized by maximum likelihood via L-BFGS-B with $50$ multi-start initializations for multi-output models and $30$ for single-output GPs.

\subsection{Population dynamics: Lotka--Volterra system}\label{sec:lv}
 
\subsubsection{Problem setup}
 
We consider the Lotka--Volterra prey-predator system \citep{Wangersky1978}:
\begin{equation}
    \begin{cases}
      p'(t) = a\,p(t) - b\,p(t)\,q(t) \\
      q'(t) = -c\,q(t) + d\,p(t)\,q(t)
    \end{cases},
    \label{eq:lv}
\end{equation}
where $p(t)$ and $q(t)$ denote the prey and predator population densities. The positive parameters $a$ and $c$ are the prey growth rate and the predator mortality rate, while $b$ and $d$ control the prey--predator interaction. This system admits a Hamiltonian structure \citep{Nutku1990}, which we linearize to obtain the constraint:
\begin{equation}\label{eq:lv_constraint}
  d\,p(t) - c\,s(t) + b\,q(t) - a\,r(t) = C_0,
\end{equation}
where $s = \log(p)$, $r = \log(q)$, and $C_0$ is a constant determined by the initial conditions. The metamodeling task consists in simultaneously predicting $Q = 4$ output fields $(p, q, r, s)$, each discretized on a uniform temporal grid of $S = 20 000$ points in $[0, 20]$, as a function of the input parameters $\mathbf{x} = (b, d) $. The inputs $b$ and $d$ are assumed to be independent and uniformly distributed over $[0.37, 0.40]$ and $[0.0, 0.06]$, respectively, with fixed parameters $a = 1.1$, $c = 0.4$ and initial conditions $(p(0), q(0)) = (1.9, 0.3)$.
 
Training sets of size $N \in \{10, 15, 30\}$ with corresponding test sets of size $N_{\text{test}} \in \{90, 85, 70\}$ such that $N+N_{\text{test}}=100$, are independently sampled from this joint uniform distribution. The constrained kernel in \texttt{Row-CMO} uses an LCM kernel $(R = 2, l = 2$, anisotropic Mat\'{e}rn $5/2$ kernels) on the latent function $\mathbf{g}$ (cf.\ Section~\ref{sec:constrained_gp}), while the LCM in \texttt{Fw-LCM} uses $R = 2$ scalar kernels (coregionalization matrices of rank at most 2, anisotropic Mat\'{e}rn $5/2$ kernels). 
All results are averaged over $10$ independent replications with randomized train and test sets, and optimizer initializations. The latent dimension $m$ varies from $1$ to $10$. All methods enforce the constraint to machine precision on the predictive mean: \texttt{Row-CMO} does so by construction, while the deductive methods achieve it through~\eqref{eq:deduction}. We therefore focus the comparison on predictive accuracy (RRMSE) as a function of $m$ and $N$. We illustrate the results on the output $p$ (prey density), and analogous patterns are observed on the other outputs and reported in~\ref{app:lv_results}.

\subsubsection{Results}
\paragraph{Alignment of Theorem~\ref{thm:dominance} with empirical RRMSE}
Figure~\ref{fig:lv_thm3} (left panel) displays the normalized dominance criterion of Theorem~\ref{thm:dominance} as a function of $m$ for $N \in \{30, 15, 10\}$. 
The criterion is positive for the first few modes indicating that column-wise on $Q$ fields captures slightly more variance and converges rapidly toward zero beyond $m \approx 4$, after which the two strategies explain essentially the same amount of variance. Notably, the convergence is faster for smaller $N$: with fewer training samples, the estimated covariance matrices become more similar, reducing the gap between the two PCA strategies. The RRMSE panels (right) confirm this alignment: at small $m$, \texttt{Row-CMO} starts with a higher RRMSE than \texttt{Col-Indep} and \texttt{Fw-Indep} variants, but the gap narrows as $m$ increases and vanishes around $m = 5$--$6$, precisely where the criterion indicates near-equivalence. 
\begin{figure}
  \centering
  \includegraphics[width=\textwidth]
  {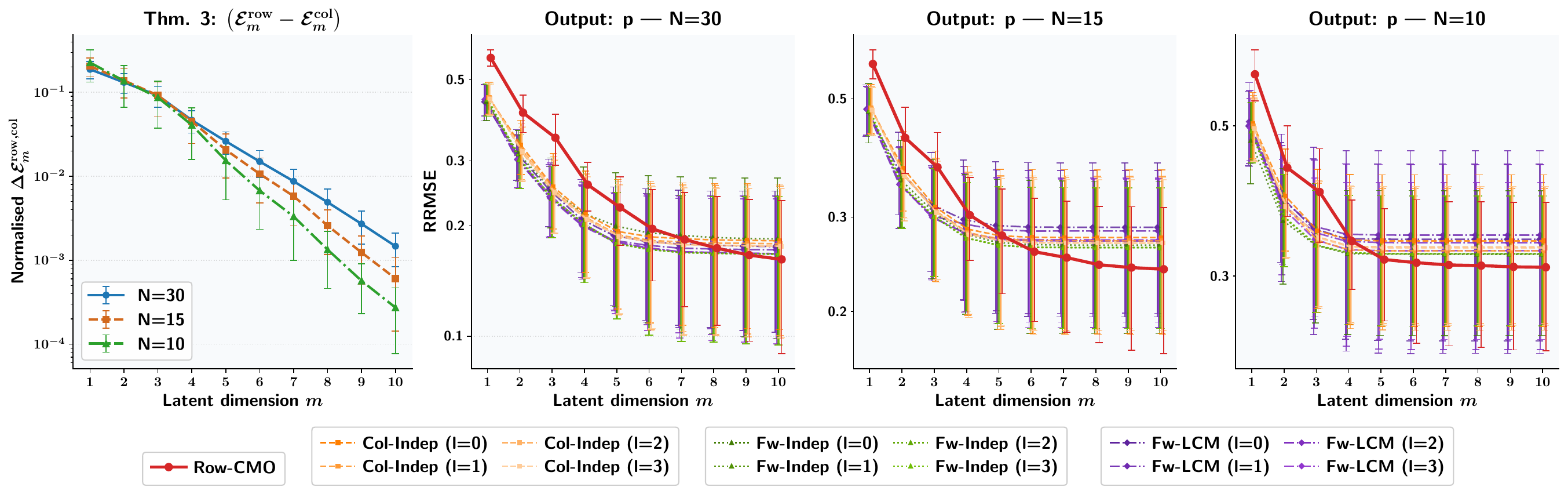}
  \caption{Left: normalized dominance criterion of Theorem~\ref{thm:dominance} as a function of $m$, for $N \in \{30, 15, 10\}$. Right: RRMSE on output $p$ for all methods and the three training set sizes ($m = 1, \ldots, 10$). Error bars represent $\pm 1$ standard deviation over $10$ seeds. The dominance criterion converges to zero beyond $m \approx 4$, and \texttt{Row-CMO} surpasses all alternatives at higher latent dimensions.}
  \label{fig:lv_thm3}
\end{figure}
\paragraph{Predictive advantage of the row-wise PCA-Constrained MOGP combination}
Although the dominance criterion indicates near-equivalence of the PCA strategies on the training data for $m \geq 5$, the RRMSE which measures predictive performance on unseen test data points to a clear advantage for our approach. Figure~\ref{fig:lv_zoom} provides a zoomed view of the RRMSE for $m \in \{5, \ldots, 10\}$, where \texttt{Row-CMO} consistently achieves the lowest RRMSE across all three training set sizes. The gap between \texttt{Row-CMO} and the alternatives widens as $N$ decreases: for $N = 10$, \texttt{Row-CMO} clearly separates from all deductive methods, whereas for $N = 30$ the differences are smaller but \texttt{Row-CMO} remains the best-performing method.
 
This discrepancy between the near-equivalence observed on training data (via the criterion) and the clear superiority observed on test data (via the RRMSE) points to a generalization advantage of the combination of row-wise PCA with constrained MOGP. Two mechanisms contribute to this effect. First, the row-wise PCA benefits from an information-pooling effect when data are scarce: the shared spatial basis is estimated from $NQ$ stacked observations, providing a more stable estimate than the $Q$ separate bases of the field-wise approach (each estimated from $N$ observations) or the single column-wise basis (estimated from $N$ observations in a $QS$-dimensional space). This improved stability translates into better out-of-sample reconstruction. Second, the constrained MOGP encodes the constraint analytically in the kernel rather than estimating the coregionalization structure from data, which reduces the effective number of hyperparameters and avoids the sensitivity to the deduction choice that penalizes the alternatives. The fact that this advantage is amplified for smaller $N$ supports the interpretation in terms of generalization capacity. 
 
We also note that \texttt{Row-CMO} outperforms not only \texttt{Col-Indep} but also \texttt{Fw-Indep} and \texttt{Fw-LCM} , despite the fact that field-wise PCA is optimal for each individual field on the training data (Theorem~\ref{thm:row}). 
 
\begin{figure}
  \centering
  \includegraphics[width=\textwidth]{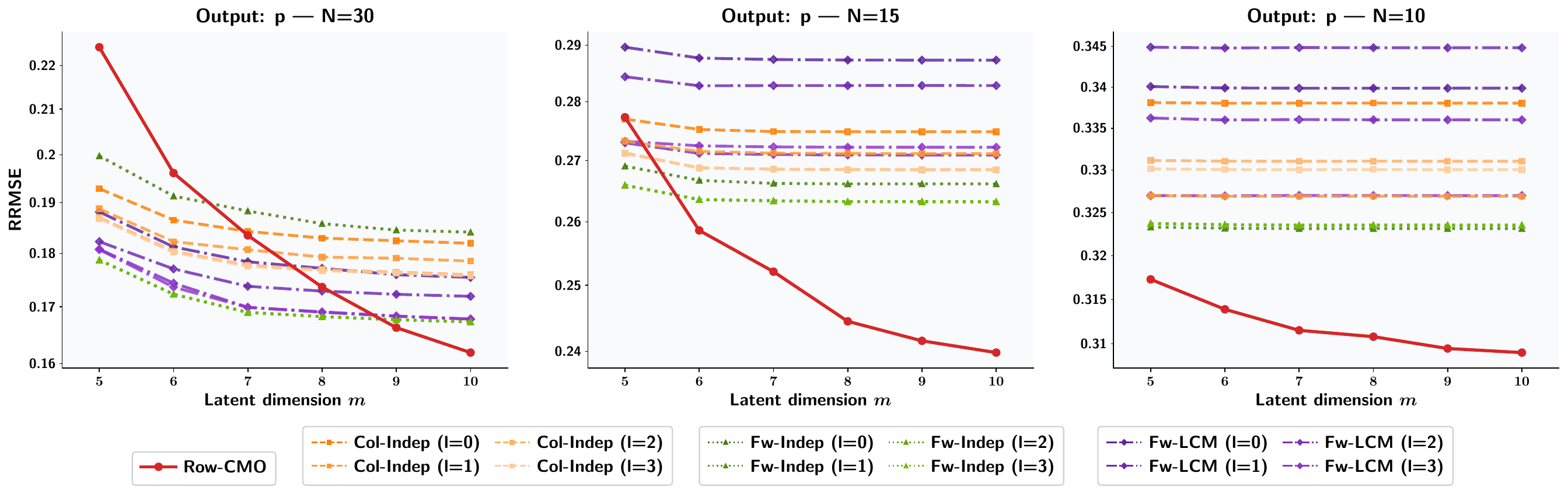}
  \caption{Zoomed view of the RRMSE on output $p$ for $m \in \{5, \ldots, 10\}$ and training set sizes $N \in \{30, 15, 10\}$. \texttt{Row-CMO} (red, solid) consistently achieves the lowest RRMSE. The spread of \texttt{Col-Indep}, \texttt{Fw-Indep}, and \texttt{Fw-LCM}  curves across deduction scenarios ($l \in \{0, 1, 2, 3\}$) illustrates the sensitivity of the deductive approaches to the arbitrary choice of deduced output.}
  \label{fig:lv_zoom}
\end{figure}
 
\paragraph{Heterogeneity on deductive approaches results}
The spread of the deductive method curves across deduction scenarios is clearly visible in both Figures~\ref{fig:lv_thm3} and~\ref{fig:lv_zoom}. 
This variability is present across all values of $N$ and all latent dimensions, confirming that it is an inherent feature of the deductive strategy. 

In summary, the Lotka--Volterra experiments confirm the three key messages of this paper: (i)~the dominance criterion of Theorem~\ref{thm:dominance} reliably predicts the relative PCA performance from the training data, (ii)~the combination of row-wise PCA with constrained MOGP exhibits a generalization advantage driven by the information-pooling effect of the shared basis and the analytical constraint encoding of the CMOGP that becomes most pronounced in data-scarce regimes, and (iii)~the deductive approaches suffer from a sensitivity to the output selection that persists across all configurations. These findings motivate the application of \texttt{Row-CMO} to the industrial-scale CFD problem that follows.

\subsection{Application to a CFD test case: uncertainty propagation of parameters of a turbulence model on a two-dimensional diffuser}
In this section, we present an application that was the primary motivation for the methods described above, which concerns the construction of metamodels for CFD fields using Gaussian processes. We propose addressing an uncertainty propagation problem within the context of the turbulent fluid flow in a diffuser, represented by a channel with a progressive widening. Its geometry is represented in Figure \ref{fig:BE_diffuser_scheme}. Experimental measurements on this diffuser has been performed in \citet{Buice2000} with a Reynolds number of 18,000. Experimentally, a fluid recirculation is observed between points $R_1$ and $R_2$ in Figure \ref{fig:BE_diffuser_scheme}. This recirculation is due to the appearance of an opposing pressure gradient (opposite to the main flow direction) caused by the widening of the cross-section, which in turn causes the boundary layer to detach. The main challenge of this application for RANS models lies in accurately predicting both the recirculation amplitude and limits, along with the velocity field throughout the entire domain.

Numerical simulations of the Buice-Eaton 2D diffuser were carried out using the solver \texttt{TrioCFD} \citep{Puscas2025}, exploring different sets of input parameters. Each simulation employed a mesh of $200,000$ triangular elements which correspond to $ S = 141,039$ mesh coordinates. To enhance computational efficiency, a domain decomposition strategy combined with a message-passing interface (MPI) was implemented for parallel computations on 11 CPU cores.
The standard $k-\varepsilon$ turbulence model \citep{Launder1974} was used, with inlet boundary conditions for velocity, turbulent kinetic energy ($k$), and dissipation rate ($\varepsilon$) obtained from a preliminary periodic simulation. A zero-pressure boundary condition was applied at the channel outlet.
For the uncertainty propagation study, the input parameters correspond to the closure coefficients of the RANS standard $k-\varepsilon$ turbulence model:

\begin{equation}\label{def:input_cfd}
\mathbf{x} = (C_{\mu},\sigma_k,\sigma_{\varepsilon}, C_{\varepsilon_1}, C_{\varepsilon_2}) \ .
\end{equation}

These five parameters govern the transport equations for $k$ and $\varepsilon$: $C_\mu$ is the proportionality constant between the eddy viscosity and $k^2/\varepsilon$, $\sigma_k$ and $\sigma_\varepsilon$ are the turbulent Prandtl numbers for $k$ and $\varepsilon$ respectively, and
$C_{\varepsilon_1}$ and $C_{\varepsilon_2}$ are the source and sink term coefficients in the equation for $\varepsilon$ respectively.
Their uncertainties are represented by probability distributions chosen based on physical considerations. Further details on the RANS modeling, the derivation of these distributions and the sampling procedure are provided in~\ref{app:uq_turb_model}.
\begin{figure}
    \centering
    \includegraphics[width=0.8\linewidth]{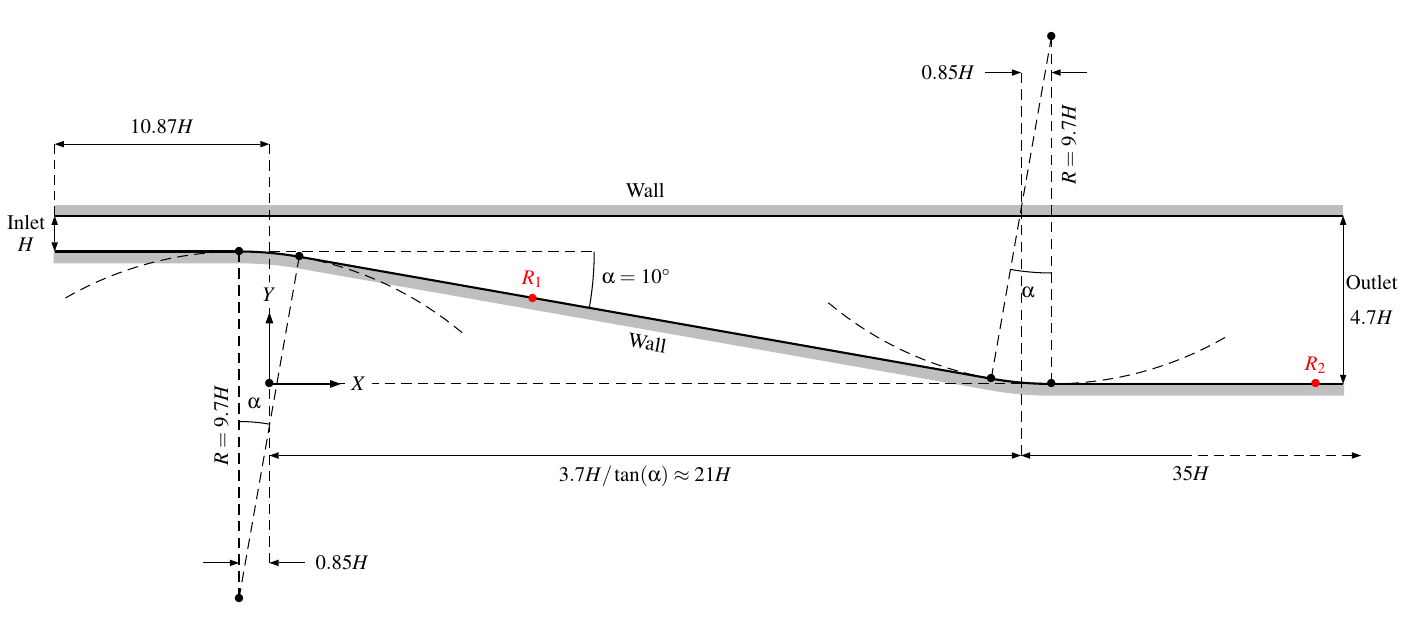}
    \caption{Scheme of the experiment carried out in \citep{Buice2000}. The wall of the channel are defined by straight lines and circular arcs of centers $O_1$ and $O_2$.}
    \label{fig:BE_diffuser_scheme}
\end{figure}\\

The quantity of interest considered as outputs in this study are the diagonal terms of the Reynolds stress tensor (RST) $\tau_{ij} = \overline{u_{i}' u_{j}'} $ , and the turbulent kinetic energy. This selection is motivated by the goal of constructing a physics-based surrogate model for the Reynolds stress tensor, as represented in linear eddy-viscosity turbulence models such as the standard $k$-$\varepsilon$ model:
\begin{equation}
    \tau_{ij} = -\nu_t \left(\frac{\partial\bar{u}_i}{\partial x_j} +\frac{\partial\bar{u}_j}{\partial x_i} \right) + \frac{2}{3}k\delta_{ij}
    \label{eqn:Linear_eddy_equation}
\end{equation}
where $\nu_t$ is the eddy viscosity. This expression is known as the Boussinesq assumption \citep{boussinesq1897theorie}, written here for incompressible flow. Combined with the periodicity in the spanwise direction (i.e $\frac{\partial\bar{u}_3}{\partial x_3} = 0$), it implies the following linear equality constraint between the diagonal components of the Reynolds stress tensor and the turbulent kinetic energy:
\begin{equation}
\tau_{11} + \tau_{22} - \frac{4}{3}k = 0.
\end{equation}
\paragraph{Training and testing datasets generation} A candidate set \((\mathbf{x}^{(i)})_{1 \leq i \leq N_c}\) of size \(N_c = 1{,}000\) is generated following the procedure described in \ref{app:uq_turb_model}. From this candidate set, a sub-sample of size \(n = 50\) is selected using a space-filling greedy algorithm that maximizes the energy distance \citep{Mak2018, Teymur2021}, thereby defining the training dataset. The testing dataset is then constructed using a Monte Carlo sample of size \(N = 100\) of the input parameters, generated according to the same procedure. Results are averaged over $10$ training runs with different optimizer initialization seeds. For the kernel configuration, the constrained kernel in \texttt{Row-CMO} uses an LCM ( $R = 2, \ell = 2$, anisotropic Mat\'{e}rn $5/2$) 
while the LCM in \texttt{Fw-LCM}  uses $R = 2$ scalar kernels with the corresponding coregionalization matrices $\mathbf{B}_1$ and $\mathbf{B}_2$ having ranks at most $1$ and $2$, respectively.
To assess the spatial distribution of prediction errors, we use the Spatial Relative Squared Error (SRSE), defined at each mesh location $\nu$ for a given test input $i$ as:
\begin{equation}\label{eq:srse}
  \text{SRSE}(\hat{y}^{(i)}(\nu),\, y^{(i)}(\nu)) = \frac{(\hat{y}^{(i)}(\nu) - y^{(i)}(\nu))^2}{\|\mathbf{y}^{(i)}\|_\infty^2}.
\end{equation}
The global RRMSE defined in~\eqref{eq:rrmse} is recovered as the joint average of the SRSE over the spatial mesh and the test set:
$
\mathrm{RRMSE}^2
= \frac{1}{N_{\mathrm{test}}\, S}
  \sum_{i=1}^{N_{\mathrm{test}}} \sum_{j=1}^{S}
  \mathrm{SRSE}\bigl(\hat{y}^{(i)}(\bm{\nu}_j), y^{(i)}(\bm{\nu}_j)\bigr).
$

\subsubsection{Results}

\paragraph{A priori diagnostic and predictive accuracy}
\begin{figure}
  \centering
  \includegraphics
  [width=\textwidth]
  {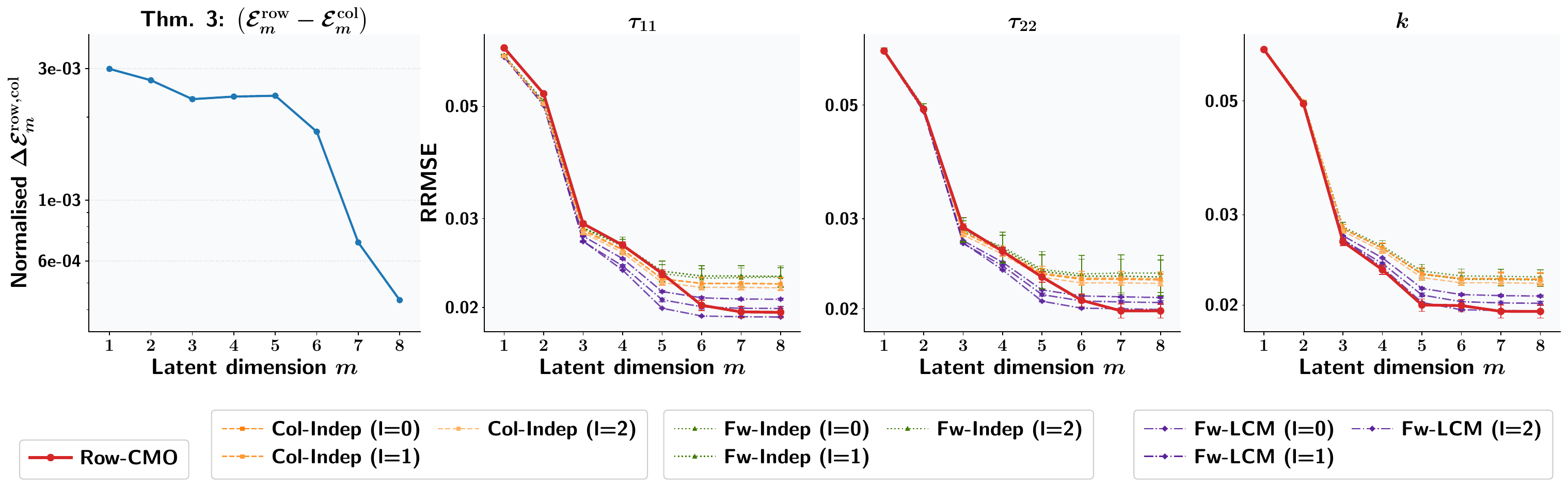}
  \caption{Left: dominance criterion of Theorem~\ref{thm:dominance} for the CFD test case, showing the normalized difference $\text{Tr}(\mathbf{D}_{\text{col},m}^2) - \text{Tr}(\mathbf{D}_{\text{row},m}^2)$ as a function of $m$. The small and decreasing values confirm near-equivalence of the two PCA strategies. Right: RRMSE on each output field ($\tau_{11}$, $\tau_{22}$, $k$) for all methods ($m = 1, \ldots, 10$). \texttt{Row-CMO} achieves the lowest RRMSE at $m \geq 6$. Error bars represent $\pm 1$ standard deviation over $10$ seeds.}
  \label{fig:cfd_thm3_rrmse}
\end{figure}
Figure~\ref{fig:cfd_thm3_rrmse} displays the dominance criterion of Theorem~\ref{thm:dominance} (left panel) alongside the RRMSE of all methods on the three output fields as a function of $m$ (right panels). The criterion serves here as a practical diagnostic: it is evaluated directly from the training data, before any GP model is trained, and indicates whether the row-wise PCA strategy is competitive with column-wise for this dataset. The values are positive but small (on the order of $10^{-3}$) and monotonically decreasing, confirming that the column-wise PCA on $Q$ fields captures only marginally more variance than row-wise, and that the two strategies are essentially equivalent for this problem. This is a direct consequence of a strong spatial correlation among the three fields.
Figure~\ref{fig:cfd_Fields} provides visual confirmation: the spatial structures of $\tau_{11}$, $\tau_{22}$, and $k$ are clearly similar, with intensity concentrated in the recirculation zone of the diffuser.
 The RRMSE panels confirm this diagnostic. At small $m$, all methods achieve very similar error levels unlike the Lotka--Volterra case where \texttt{Row-CMO} started above the alternatives. This alignment is expected: when the dominance criterion indicates near-equivalence from the outset, the PCA strategies contribute similarly to the reconstruction quality, regardless of whether row-wise, column-wise, or field-wise is used. As $m$ increases beyond $m \approx 5$, \texttt{Row-CMO} progressively separates from the alternatives and achieves the lowest RRMSE on all three outputs at $m = 8$. This pattern is consistent with the observations on the Lotka--Volterra problem: once the PCA strategies are equivalent, the generalization advantage of \texttt{Row-CMO}, the information-pooling effect of the shared basis and the analytical constraint encoding of the CMOGP drives the performance gap, especially in this regime where only $N = 50$ training points are available.
\begin{figure}
  \centering
  \includegraphics[width=\textwidth]{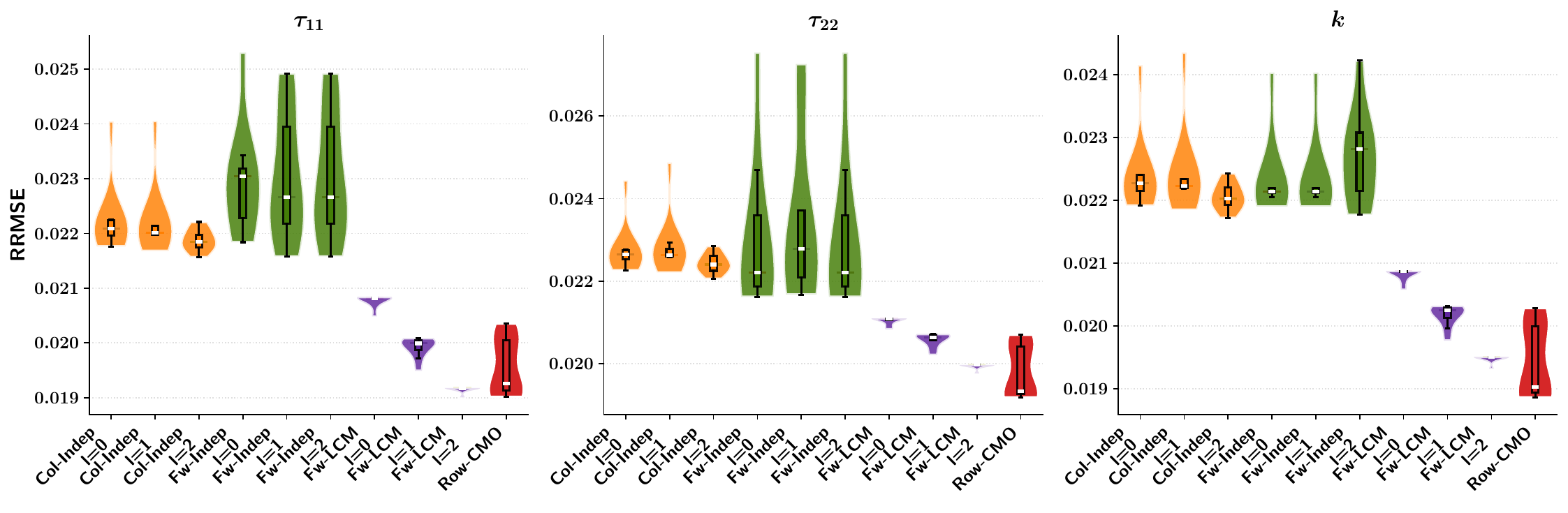}
  \caption{RRMSE at $m = 8$ for each output field ($\tau_{11}$, $\tau_{22}$, $k$) across all methods and deduction scenarios. \texttt{Row-CMO} achieves the lowest mean RRMSE with the smallest variability. The spread across deduction scenarios within each family (\texttt{Col-Indep}, \texttt{Fw-Indep}, \texttt{Fw-LCM} ) and the larger error bars of the deductive methods illustrate the residual sensitivity to the output selection, even in a regime of strong inter-field correlation.}
  \label{fig:cfd_bar}
\end{figure}
Figure~\ref{fig:cfd_bar} provides a detailed view at $m = 8$: \texttt{Row-CMO} achieves the lowest mean RRMSE on each output with a narrower variability across seeds. The deductive methods exhibit a tighter grouping of their variants than in the Lotka--Volterra case, which is expected given the strong inter-field correlation. Nevertheless, a residual sensitivity to the deduction scenario persists: the spread across variants within each deductive family (\texttt{Col-Indep}, \texttt{Fw-Indep}, \texttt{Fw-LCM} ) remains visible. This confirms on a realistic industrial problem, that the sensitivity of the deductive approach to the arbitrary output selection does not vanish even when the fields are strongly correlated.

\paragraph{Spatial distribution of predictions and errors}
Figure~\ref{fig:cfd_Fields} illustrates the predictive quality of \texttt{Row-CMO} for a representative test input, using $M = 8$ latent dimensions. The left two columns compare the simulated and predicted fields for $\tau_{11}$, $\tau_{22}$, and $k$: the model accurately reproduces the spatial structure of all three fields, including the intensity and extent of the recirculation zone. The third column displays the predictive standard deviation, which is highest in the recirculation region where the flow dynamics are most sensitive to the input parameters and lowest in the upstream channel, reflecting the model's ability to identify regions of high and low predictive confidence. The rightmost column shows the SRSE: the prediction errors are localized in the same recirculation region, with magnitudes on the order of $10^{-4}$ to $10^{-3}$, confirming the overall accuracy of the model.
 
\begin{figure}
  \centering
  \includegraphics[width=\textwidth]{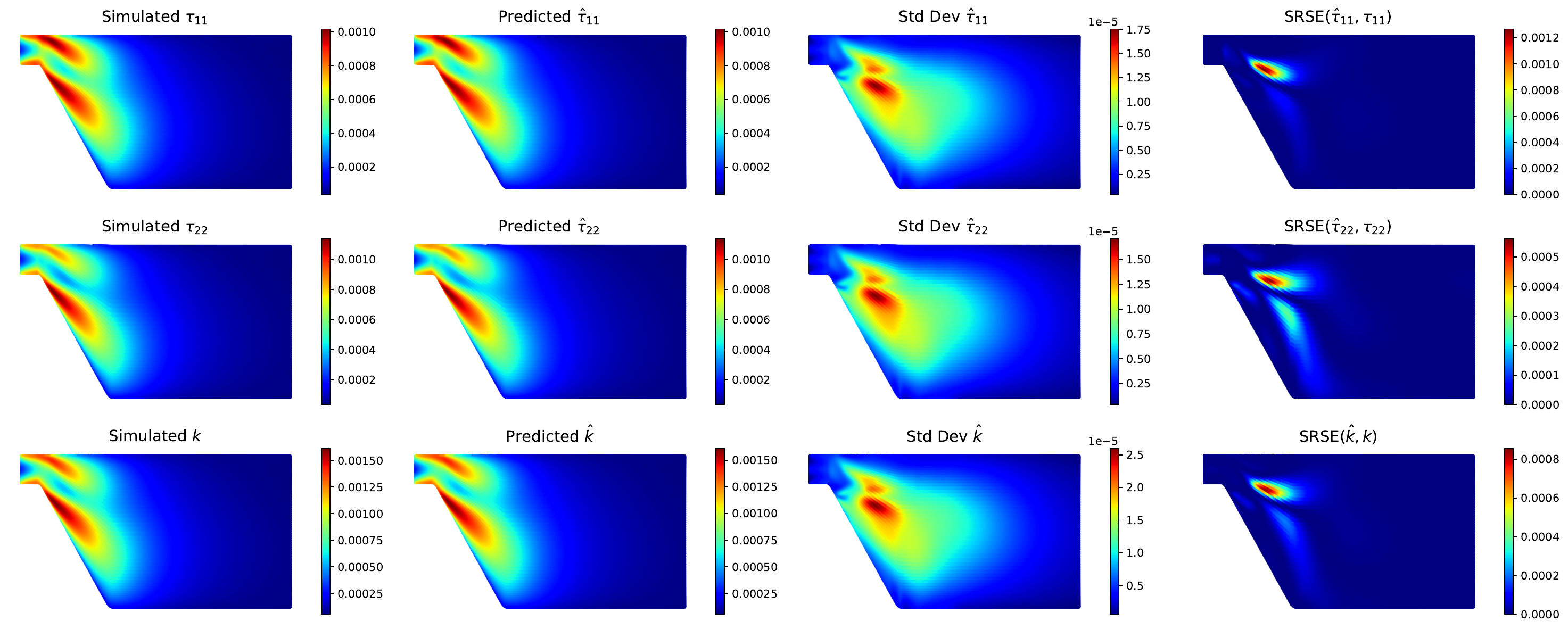}
  \caption{\texttt{Row-CMO} predictions on the CFD test case for a representative test input ($M = 8$). From top to bottom: $\tau_{11}$, $\tau_{22}$, $k$. Left to right: simulated field, predicted field, predictive standard deviation, and SRSE. The model accurately captures the spatial structure. Errors and uncertainty are concentrated in the recirculation zone of the diffuser.}
  \label{fig:cfd_Fields}
\end{figure}

These results confirm that the proposed approach is well suited to this class of problems where the output fields share strong spatial correlations. In this regime, the row-wise PCA incurs a negligible reconstruction cost, and the generalization advantage of \texttt{Row-CMO} enables it to outperform the deductive alternatives at moderate-to-high latent dimensions, while guaranteeing strict constraint satisfaction.
\section{Conclusion}\label{sec:conclusion}
 
We propose a framework for the simultaneous prediction of multiple high-dimensional physical fields governed by linear equality constraints. The approach combines a row-wise PCA strategy, which uniquely preserves the constraint structure in the latent space with constrained multi-output GP regression based on a kernel parametrization that analytically encodes the constraint.
The motivation for joint constrained modeling, as opposed to a trivial deductive approach, was established through a benchmark on scalar-output functions. This benchmark demonstrated that the deductive strategy introduces a sensitivity to the arbitrary choice of which output to deduce, affecting both predictive accuracy and uncertainty quantification, and that this sensitivity persists regardless of the regression method used for the remaining outputs. The constrained MOGP avoids these issues by treating all outputs symmetrically and encoding the constraint directly in the covariance structure.
To assess the reconstruction cost of the constraint-preserving row-wise PCA relative to the more common column-wise and field-wise strategies, we provided a theoretical analysis based on perturbation bounds and an exact dominance criterion. 
The framework was validated on a population dynamics problem based on the Lotka--Volterra system and on an industrial CFD application involving the prediction of Reynolds stress tensor components under the incompressibility constraint. In both cases, our method achieved competitive or superior predictive performance compared to deductive alternatives combining column-wise, field-wise, or LCM-based regression with output deduction. The generalization advantage of the proposed method was most pronounced in data-scarce regimes, where the information-pooling effect of the shared spatial basis and the analytical constraint encoding jointly contributed to improved out-of-sample accuracy.
 
Several directions for future work can be identified. First, the parametrization approach in latent space should make it possible to handle settings where the different outputs fields are not observed at the same input locations, by exploiting the constraint structure to couple the fields. Second, extending the methodology to spatio-temporal data from physical simulations, following the approach of \citet{Mak2018} is a natural perspective. Third, while the present framework handles linear equality constraints, many physical systems involve also inequality constraints (e.g, positivity of energy or concentration fields). Extending the parametrization approach to such settings would broaden the applicability of the method. Finally, for problems exhibiting strong nonlinearities in the output fields, the linear PCA step may become a bottleneck, and coupling constrained GP regression with nonlinear dimensionality reduction techniques such as autoencoders or kernel PCA, represents a promising but challenging research direction.

\section{Declaration of generative AI and AI-assisted technologies in the manuscript preparation process}
During the preparation of this work the authors used Claude in order to improve software development for the numerical experiments. After using Claude, the authors reviewed and edited the content as needed and take full responsibility for the content of the published article.

\section{Acknowledgments} 
This work has received funding from the France 2030 NumPEx Exa-MA (ANR-22-EXNU-0002) project managed by the French National Research Agency (ANR) 

\newpage
\appendix


\section{Proofs of the theoretical results}\label{app:proofs}

We recall the Davis--Kahan $\sin\Theta$ theorem \citep{DavisKahan1970, StewartSun1990} in the form used throughout our proofs.

\begin{theorem}[Davis--Kahan]\label{thm:DK}
Let $\boldsymbol{\Sigma}_A, \boldsymbol{\Sigma}_B \in \mathbb{R}^{S \times S}$ be symmetric matrices with eigenvalues $\lambda_1^A \geq \cdots \geq \lambda_S^A$ and $\lambda_1^B \geq \cdots \geq \lambda_S^B$. Let $\mathbf{A} = (\mathbf{a}_1, \ldots, \mathbf{a}_m)$ and $\mathbf{B} = (\mathbf{b}_1, \ldots, \mathbf{b}_m)$ be the matrices of the first $m$ eigenvectors of $\boldsymbol{\Sigma}_A$ and $\boldsymbol{\Sigma}_B$ respectively, and let $\delta_A = \lambda_m^A - \lambda_{m+1}^A > 0$. Then:
\begin{equation}\label{eq:DK}
  \|\sin\Theta(\mathbf{A}, \mathbf{B})\|_F \leq \frac{\min\bigl\{\sqrt{m}\,\|\boldsymbol{\Sigma}_B - \boldsymbol{\Sigma}_A\|_{\text{op}},\;\|\boldsymbol{\Sigma}_B - \boldsymbol{\Sigma}_A\|_F\bigr\}}{\delta_A},
\end{equation}
where $\Theta(\mathbf{A}, \mathbf{B})$ denotes the diagonal matrix of principal angles between the column spaces of $\mathbf{A}$ and $\mathbf{B}$.
\end{theorem}

We also use the relation between the projector distance and the principal angles: for projectors $\mathbf{P}_A = \mathbf{A}\mathbf{A}^\top$ and $\mathbf{P}_B = \mathbf{B}\mathbf{B}^\top$ of rank $m$,
\begin{equation}\label{eq:proj_angle}
  \|\mathbf{P}_A - \mathbf{P}_B\|_F = \sqrt{2}\,\|\sin\Theta(\mathbf{A}, \mathbf{B})\|_F.
\end{equation}

\subsection{Proof of Theorem~\ref{thm:row} (Row-wise excess error bound)}

\begin{proof}\:  Using the trace formulation and the idempotence of projectors:
\begin{align}
  \Delta E_{k,m}^{(\text{row})} &= \|\mathbf{Y}_k(\mathbf{I}_S - \mathbf{P}_{\text{row},m})\|_F^2 - \|\mathbf{Y}_k(\mathbf{I}_S - \mathbf{P}_{k,m})\|_F^2 \nonumber \\
  &= \text{Tr}\bigl(N\boldsymbol{\Sigma}_k(\mathbf{P}_{k,m} - \mathbf{P}_{\text{row},m})\bigr). \label{eq:proof_step1}
\end{align}
By the Cauchy--Schwarz inequality for the trace inner product:
\begin{equation}\label{eq:proof_CS}
  \Delta E_{k,m}^{(\text{row})} \leq N\,\|\boldsymbol{\Sigma}_k\|_F\,\|\mathbf{P}_{k,m} - \mathbf{P}_{\text{row},m}\|_F.
\end{equation}
Now, we apply Theorem~\ref{thm:DK} with $\boldsymbol{\Sigma}_A = \boldsymbol{\Sigma}_k$ and $\boldsymbol{\Sigma}_B = \boldsymbol{\Sigma}_{\text{row}}$. The perturbation is:
\[
  \boldsymbol{\Delta\Sigma}_k = \boldsymbol{\Sigma}_{\text{row}} - \boldsymbol{\Sigma}_k = \frac{1}{Q}\sum_{j=1}^Q (\boldsymbol{\Sigma}_j - \boldsymbol{\Sigma}_k).
\]
Using~\eqref{eq:proj_angle}, the projector distance satisfies:
\begin{equation}\label{eq:proof_proj}
  \|\mathbf{P}_{k,m} - \mathbf{P}_{\text{row},m}\|_F = \sqrt{2}\,\|\sin\Theta(\mathbf{V}_{k,m}, \mathbf{V}_{\text{row},m})\|_F \leq \frac{\sqrt{2}\,\min\bigl\{\sqrt{m}\,\|\boldsymbol{\Delta\Sigma}_k\|_{\text{op}},\;\|\boldsymbol{\Delta\Sigma}_k\|_F\bigr\}}{\delta_k}.
\end{equation}
Substituting~\eqref{eq:proof_proj} into~\eqref{eq:proof_CS}:
\[
  \Delta E_{k,m}^{(\text{row})} \leq \frac{2N\sqrt{2}\,\|\boldsymbol{\Sigma}_k\|_F}{\delta_k} \cdot \min\bigl\{\sqrt{m}\,\|\boldsymbol{\Delta\Sigma}_k\|_{\text{op}},\;\|\boldsymbol{\Delta\Sigma}_k\|_F\bigr\}. \qedhere
\]
\end{proof}

\subsection{Proof of Theorem~\ref{thm:col} (Column-wise excess error bound)}

\begin{proof}
The argument is analogous to Theorem~\ref{thm:row} but operates in the observation space. Using the equivalence of left and right projections for the SVD ($\mathbf{Y}_k\mathbf{P}_{k,m} = \boldsymbol{\Pi}_{k,m}\mathbf{Y}_k$), the excess error becomes:
\begin{equation}
  \Delta E_{k,m}^{(\text{col})} = \|(\mathbf{I}_N - \boldsymbol{\Pi}_{\text{col},m})\mathbf{Y}_k\|_F^2 - \|(\mathbf{I}_N - \boldsymbol{\Pi}_{k,m})\mathbf{Y}_k\|_F^2 = \text{Tr}\bigl(S\mathbf{K}_k(\boldsymbol{\Pi}_{k,m} - \boldsymbol{\Pi}_{\text{col},m})\bigr).
\end{equation}
By Cauchy--Schwarz: 
\begin{equation}
    \Delta E_{k,m}^{(\text{col})} \leq S\,\|\mathbf{K}_k\|_F\,\|\boldsymbol{\Pi}_{k,m} - \boldsymbol{\Pi}_{\text{col},m}\|_F
    \label{eqn:col_CS}
\end{equation}
Then we apply Theorem~\ref{thm:DK}  with $\boldsymbol{\Sigma}_A = \mathbf{K}_k$ and $\boldsymbol{\Sigma}_B = \mathbf{K}_{\text{col}}$. The perturbation is $\boldsymbol{\Delta K}_k = \mathbf{K}_{\text{col}} - \mathbf{K}_k = \frac{1}{Q}\sum_j (\mathbf{K}_j - \mathbf{K}_k)$, and:
\begin{equation}
     \|\boldsymbol{\Pi}_{k,m} - \boldsymbol{\Pi}_{\text{col},m}\|_F \leq \frac{\sqrt{2}\,\min\bigl\{\sqrt{m}\,\|\boldsymbol{\Delta K}_k\|_{\text{op}},\;\|\boldsymbol{\Delta K}_k\|_F\bigr\}}{\delta_k^{(\text{col})}}.
     \label{eq:Pik_bound}
\end{equation}
Substituting \eqref{eq:Pik_bound} into \eqref{eqn:col_CS} yields:
\[
  \Delta E_{k,m}^{(\text{col})} \leq \frac{2S\sqrt{2}\,\|\mathbf{K}_k\|_F}{\delta_k^{(\text{col})}} \cdot \min\bigl\{\sqrt{m}\,\|\boldsymbol{\Delta K}_k\|_{\text{op}},\;\|\boldsymbol{\Delta K}_k\|_F\bigr\}. \qedhere
\]
\end{proof}

\subsection{Proof of Theorem~\ref{thm:dominance} (Row-wise vs.\ column-wise)}

\begin{proof}
Both global errors decompose as total energy minus explained variance. For the row-wise approach:
\[
  E_m^{(\text{row})} = \|\mathbf{Y}_{\text{row}}\|_F^2 - \text{Tr}(\mathbf{D}_{\text{row},m}^2),
\]
and for the column-wise approach:
\[
  E_m^{(\text{col})} = \|\mathbf{Y}_{\text{col}}\|_F^2 - \text{Tr}(\mathbf{D}_{\text{col},m}^2).
\]
Since $\mathbf{Y}_{\text{row}}$ and $\mathbf{Y}_{\text{col}}$ contain exactly the same elements rearranged, their Frobenius norms coincide: $\|\mathbf{Y}_{\text{row}}\|_F^2 = \|\mathbf{Y}_{\text{col}}\|_F^2$. Subtracting yields
\[
  E_m^{(\text{col})} - E_m^{(\text{row})} = \text{Tr}(\mathbf{D}_{\text{row},m}^2) - \text{Tr}(\mathbf{D}_{\text{col},m}^2).
\]
Since $\text{Tr}(\mathbf{D}_{\text{row},m}^2) = NQ\sum_{p=1}^m \lambda_p^{(\text{row})}$ and $\text{Tr}(\mathbf{D}_{\text{col},m}^2) = QS\sum_{p=1}^m \lambda_p^{(\text{col})}$, the equivalence~\eqref{eq:dominance_eig} follows.
\end{proof}

\subsection{Proof of Corollary~\ref{cor:zero}}

\begin{proof}
If $\mathbf{V}_j = \mathbf{V}_0$ for all $j$, then $\boldsymbol{\Sigma}_j = \mathbf{V}_0(\frac{1}{N}\mathbf{D}_j^2)\mathbf{V}_0^\top$ and $\boldsymbol{\Sigma}_{\text{row}} = \mathbf{V}_0(\frac{1}{NQ}\sum_j \mathbf{D}_j^2)\mathbf{V}_0^\top$. Since $\frac{1}{NQ}\sum_j \mathbf{D}_j^2$ is diagonal, $\mathbf{V}_0$ is also the eigenvector matrix of $\boldsymbol{\Sigma}_{\text{row}}$. Truncating at rank $m$ then yields $\mathbf{V}_{\text{row},m} = \mathbf{V}_{0,m} = \mathbf{V}_{k,m}$, hence $\mathbf{P}_{\text{row},m} = \mathbf{P}_{k,m}$ and $\Delta E_{k,m}^{(\text{row})} = 0$. The column-wise case is analogous with $\mathbf{U}$ and $\mathbf{K}$.
\end{proof}


\section[
  Uncertainty quantification of the standard k-epsilon turbulence model parameters
]{%
  Uncertainty quantification of the standard \ensuremath{k\text{–}\varepsilon} turbulence model parameters
}\label{app:uq_turb_model}

The standard $k$–$\varepsilon$ turbulence model \citep{Launder1974} is widely used in the Computational Fluid Dynamics (CFD) community. It provides reliable predictions for a wide range of flows, including industrial applications, while maintaining reasonable computational cost. The principal hypotheses underlying the model are as follows:
\begin{itemize}
    \item Turbulence has the same effect as viscous forces. Hence, an eddy viscosity coefficient $\nu_t$ is introduced. It varies spatially depending on the local flow conditions.
    \item The eddy viscosity can be expressed as a function of turbulent kinetic energy $k$ and turbulent dissipation rate $\varepsilon$. A dimensional analysis gives:
    \begin{equation}\label{eq:nu_t_def}
        \nu_t = C_{\mu}\frac{k^2}{\varepsilon}
    \end{equation}
    where $C_{\mu}$ is a constant calibrated on experimental data.
\end{itemize}
The turbulent kinetic energy is modeled by a transport equation that can be derived exactly from the Reynolds decomposition of the flow velocity $u_i = \bar{u}_i + u'_i$. The exact transport equation for the turbulent dissipation rate is highly complex and is therefore represented in a simplified modeled form. The equations of the standard $k$–$\varepsilon$ model are written as follows:

\begin{equation}
\left\{
\begin{array}{l}
\dfrac{\partial k}{\partial t} + \bar{u}_j \dfrac{\partial k}{\partial x_j} = \dfrac{\partial}{\partial x_j} \left[ \left( \nu + \dfrac{\nu_t}{\sigma_k} \right) \dfrac{\partial k}{\partial x_j} \right] + P_k - \varepsilon\\\\
\dfrac{\partial \varepsilon}{\partial t} + \bar{u} \dfrac{\partial \varepsilon}{\partial x_j} = \dfrac{\partial}{\partial x_j} \left[ \left( \nu + \dfrac{\nu_t}{\sigma_\varepsilon} \right) \dfrac{\partial \varepsilon}{\partial x_j} \right] + \dfrac{\varepsilon}{k} (C_{\varepsilon_1} P_k - C_{\varepsilon_2} \varepsilon)
\end{array}
\right.
\end{equation}

where

\begin{equation}
P_k = - \tau_{ij} \frac{\partial \bar{u}_i}{\partial x_j}
\end{equation}

is the production of turbulent kinetic energy. In the equation for the turbulent dissipation rate, three parameters have to be calibrated:
\begin{itemize}
    \item the turbulent Prandtl number for dissipation $\sigma_{\varepsilon}$;
    \item the coefficient for the source term $C_{\varepsilon 1}$;
    \item the coefficient for the sink term $C_{\varepsilon 2}$.
\end{itemize}

Ultimately, the model contains five parameters requiring experimental or semi-analytical calibration, with their standard values reported in Table \ref{tab:keps_values}.
\begin{table} 
\centering
\begin{tabular}{ c|c|c|c|c } 
 $C_{\mu}$ & $\sigma_k$ & $\sigma_{\varepsilon}$ & $C_{\varepsilon_1}$ & $C_{\varepsilon_2}$ \\
 \hline
 0.09 & 1.00 & 1.30 & 1.44 & 1.92 \\ 
\end{tabular}
\caption{Usual values of the parameters of the standard $k-\varepsilon$ model.}
\label{tab:keps_values}
\end{table}

\subsection[C\_mu calibration]{Calibration of $C_\mu$}

In the case of a parallel shear flow, the $k-\varepsilon$ model equations can be simplified to yield the following expression of $C_{\mu}$:
\begin{equation}\label{eq:def_c_mu}
C_{\mu} = \frac{\varepsilon}{P_k}\left(\frac{\overline{u'_1u'_2}}{k}\right)^2 \ .
\end{equation}
Assuming local equilibrium between production and dissipation, it follows that $P_k = \varepsilon$. Therefore, $C_{\mu}$ can be computed as:
\begin{equation}\label{eq:reynolds_c_mu}
    C_{\mu} = \left(\frac{2 \tau_{12}}{\tau_{11}+\tau_{22}+\tau_{33}}\right)^2 \ .
\end{equation}

We can thus determine the value of $C_{\mu}$ and its uncertainties using the experimental data and information from \citet{Buice2000}.  
\begin{figure}
     \centering
     \begin{subfigure}[b]{0.7\textwidth}
         \centering
         \includegraphics[width=\textwidth]{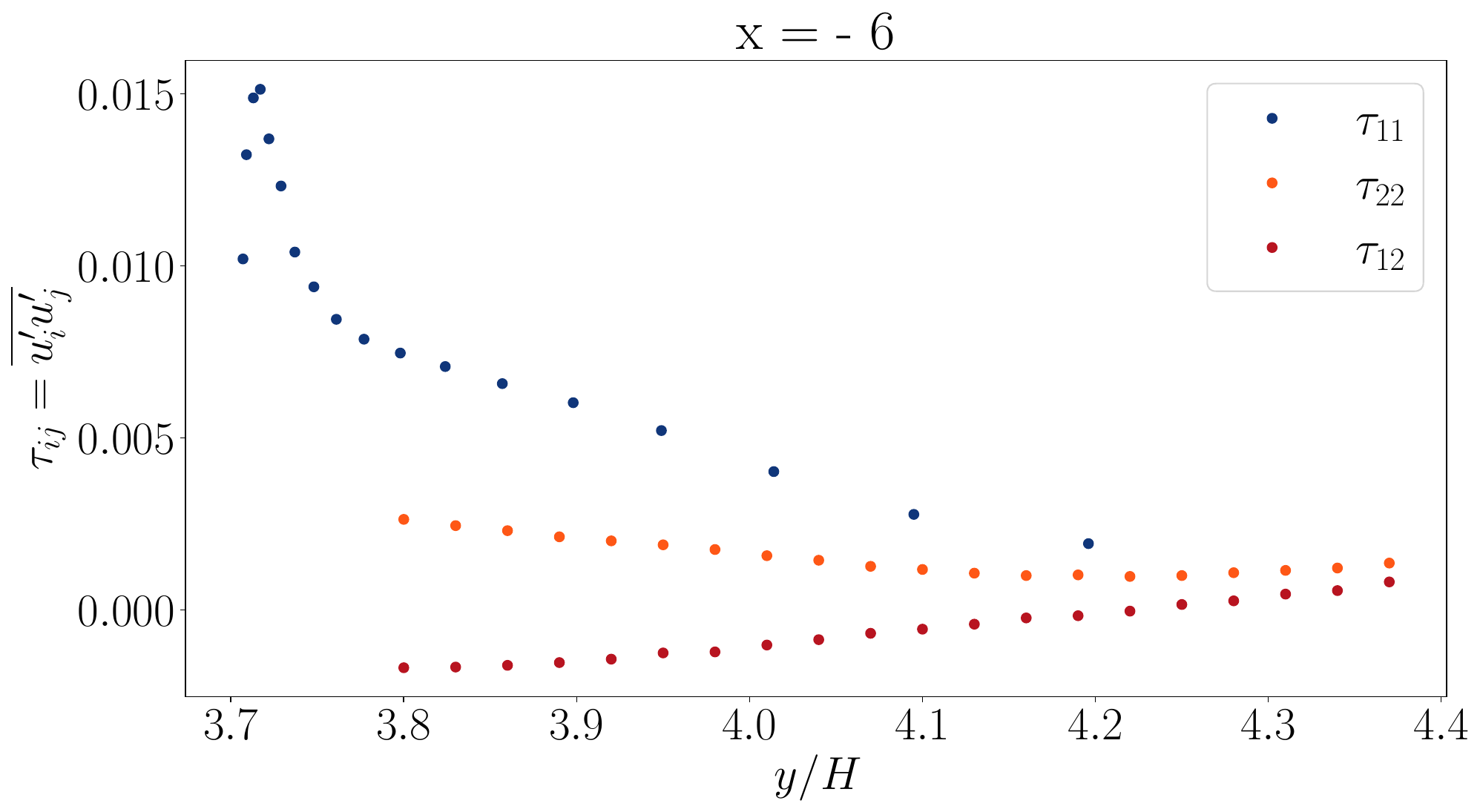}
         \caption{Experimental data from \citep{Buice2000}.}
         \label{fig:exp_data_Buice}
     \end{subfigure}
     \hfill
     \begin{subfigure}[b]{0.7\textwidth}
         \centering
         \includegraphics[width=\textwidth]{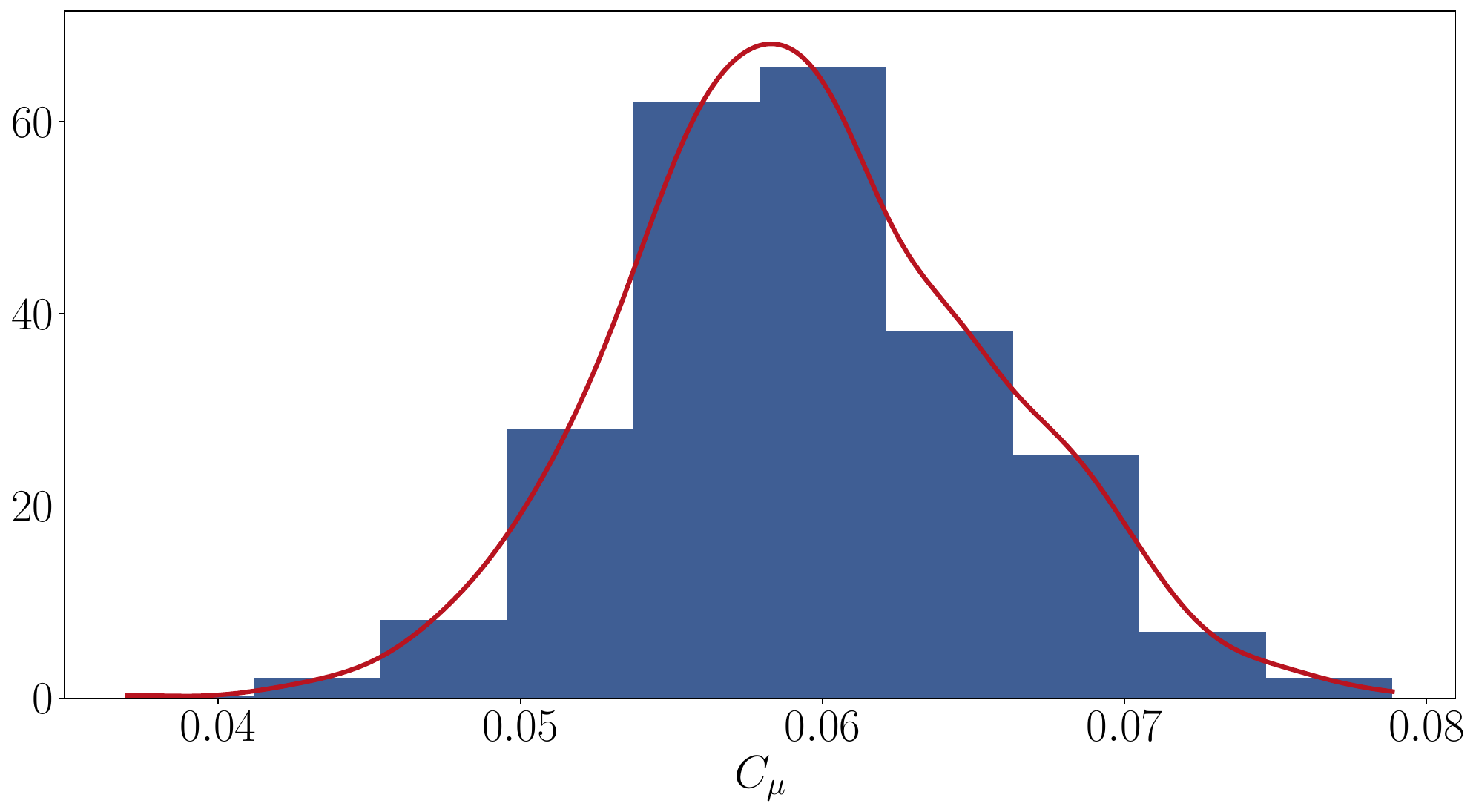}
         \caption{Probability density function (PDF) estimate of $C_{\mu}$.}
         \label{fig:pdf_C_mu}
     \end{subfigure}
     \caption{Experimental data of \citet{Buice2000} and probability density estimate of $C_{\mu}$}
\end{figure}
Figure~\ref{fig:exp_data_Buice} shows measured values of the Reynolds stress tensor in the channel upstream of the diffuser, indexed by their position along the axis perpendicular to the flow direction. At $y/H = 4$, which is close to the channel wall, the local equilibrium hypothesis may hold. A nominal value of $C_{\mu} = 0.06$ is obtained at this location, which is close to the classical value of $C_{\mu} = 0.09$.
\newline

In \citet{Buice2000}, information about the measurement uncertainties for the different component Reynolds stress tensor are reported as follows: $4\%$ for $\tau_{11}$, $5\%$ for $\tau_{22}$ and $10\%$ for $\tau_{12}$. Note that $\tau_{33}$ is not given in the Buice-Eaton experimental campaign data. An \textit{ad-hoc} value of $(\tau_{11} + \tau_{22})/2$ is used with an assigned uncertainty of $4\%$. Accordingly, the probability distributions of the $\tau_{ij}$ are assumed to be Gaussian, such that:
\begin{equation}\label{eq:pdf R_ij}
    \tau_{ij} \sim \mathcal{N}\left(\tau_{ij, {\rm mes}}, \frac{\Delta \tau_{ij} \times \tau_{ij, {\rm mes}}}{2}\right)
\end{equation}
where $\tau_{ij, {\rm mes}}$ are the measured values for $y/H=4$ specified in Table \ref{tab:Rij_values}, and $\Delta \tau_{ij}$ the uncertainties values for the corresponding component of the Reynolds stress tensor.
\begin{table}
\centering
\begin{tabular}{ c|c|c|c } 
 $\tau_{11}$ & $\tau_{22}$ & $\tau_{33}$ & $\tau_{12}$ \\
 \hline
 $4.020\text{e-}3$ & $1.576\text{e-}3$ & $2.798\text{e-}3$ & $1.020\text{e-}3$\\ 
\end{tabular}
\caption{Measured values for the Reynolds stress tensor component for $y/H = 4$.}
\label{tab:Rij_values}
\end{table}
Figure~\ref{fig:pdf_C_mu} shows the nonparametric estimate of the PDF of $C_{\mu}$, based on a simulated sample of 1,000 draws from the Gaussian distributions of the components $\tau_{ij}$, which are then used to compute $C_{\mu}$.
\subsection[C\_epsilon2 calibration]{$C_{\varepsilon_2}$ calibration}
The parameter $C_{\varepsilon_2}$ is calibrated for an isotropic, homogeneous turbulence regime in which the turbulent kinetic energy decays according to a power law $k(t) \propto t^{-\beta}$. In this case, the spatial gradients of the flow velocity $k$ and $\varepsilon$ are zero. From the standard $k$–$\varepsilon$ model, it follows that:
\begin{equation}\label{eq:def_c_eps_2}
    C_{\varepsilon_2} = \frac{\beta +1}{\beta}
\end{equation}
The probability density function of $\beta$ is inferred from Figure 1 of \cite{Meldi_Sagaut_2012}, which shows a histogram of a sample of 650 $\beta$ values obtained from experimental data and Direct Numerical Simulations \citep{Moin1998}. Based on this information, we assume that the power-law exponent $\beta$ follows a uniform distribution such that:
\begin{equation} \label{eq:beta_pdf}
    \beta \sim \mathcal{U}([0.95, 1.25]) \ .
\end{equation}
Figure \ref{fig:C_eps2_pdf} shows an histogram and a Kernel Density Estimation (KDE) using a Gaussian kernel on a 1,000-sized sample of $C_{\varepsilon_2}$ values, determined using a 1,000-sized sample of $\beta$ values using the distribution defined in Equation \eqref{eq:beta_pdf} and the relationship between $C_{\varepsilon_2}$ and $\beta$ given in Equation \eqref{eq:def_c_eps_2}.
\begin{figure}
    \centering
    \includegraphics[width=0.8\linewidth]{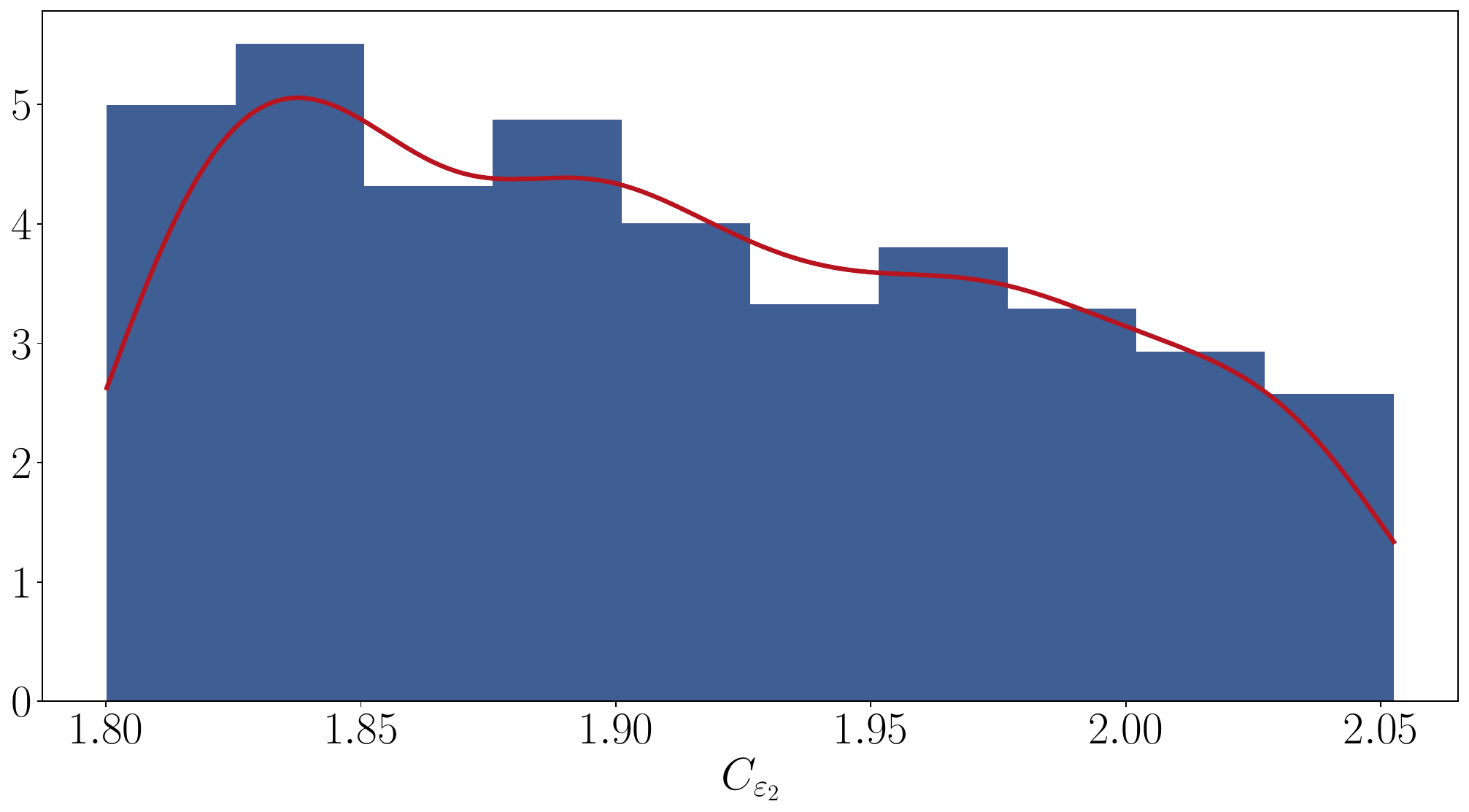}
    \caption{Histogram and gaussian KDE of the distribution of $C_{\varepsilon_2}$ using a 1,000-sized sample of $\beta$ values.}
    \label{fig:C_eps2_pdf}
\end{figure}

\subsection[Prandtl numbers sigma\_k and sigma\_epsilon]{Prandtl numbers $\sigma_k$ and $\sigma_{\varepsilon}$}
The parameters $\sigma_k$ and $\sigma_{\varepsilon}$ are assigned a $10\%$ uncertainty, as reported in \citet{Foken2006}. Additionally, the constraint $2\sigma_k - \sigma_{\varepsilon} \leq 1$ must be enforced to ensure the correct behavior of the standard $k$–$\varepsilon$ model at the turbulent/nonturbulent interface \citep{Cazalbou1994}. Accordingly, $\sigma_k$ and $\sigma_{\varepsilon}$ are assumed to be uniformly distributed over the ranges $[\sigma_k^* \times 0.9, \sigma_k^* \times 1.1]$ and $[\sigma_{\varepsilon}^* \times 0.9, \sigma_{\varepsilon}^* \times 1.1]$, respectively. From an i.i.d. sample $(\sigma_{k,i}, \sigma_{\varepsilon,i})_{1 \leq i \leq N}$, only those samples satisfying $2\sigma_{k,i} - \sigma_{\varepsilon,i} \leq 1$ are retained.

\subsection[C\_epsilon1 calibration]{$C_{\varepsilon_1}$ calibration}
Considering the fluid flow in the boundary layer and local equilibrium ($\varepsilon = P_k$), we obtain the following equation between $C_{\varepsilon_1}$ and $C_{\varepsilon_2}$
\begin{equation}\label{eq:C_eps_1_def}
    C_{\varepsilon_1} =  C_{\varepsilon_2} - \frac{\kappa^2}{\sigma_{\varepsilon}\sqrt{C_{\mu}}} \ ,
\end{equation}
where $\kappa$ is the Von Karman constant. A dataset of $n=8$ values are obtained from \citet{Foken2006}. A Gaussian is fitted on this dataset with its parameters $(m_{\kappa}, \sigma_{\kappa}^2)$ estimated empirically. 

\subsection[Sampling procedure for the standard k-epsilon turbulence model parameters]{Sampling procedure for the standard $k-\varepsilon$ turbulence model parameters}\label{sec:u_prop}
After properly assessing the sources of uncertainty in the $k$–$\varepsilon$ model, Algorithm \ref{alg:samp} outlines the procedure for sampling values of the standard $k$–$\varepsilon$ parameter vector $\mathbf{x} = (C_{\mu}, \sigma_k, \sigma_{\varepsilon}, C_{\varepsilon_1}, C_{\varepsilon_2})$.
\begin{algorithm}
	\caption{Sampling algorithm of the standard $k-\varepsilon$ turbulence model parameters.}
	\label{alg:samp} 
    \textit{Step 1:} Sample $(\sigma_k, \sigma_{\varepsilon})$ from their respective Uniform distribution conditioned by $2\sigma_k - \sigma_{\varepsilon} \leq 1$. \\ 
    \textit{Step 2:} Sample $\kappa \sim \mathcal{N}(m_{\kappa}, \sigma_{\kappa}^2)$. \\ 
    \textit{Step 3:} Sample $\beta \sim \mathcal{U}([0.95, 1.25])$. \\
    \textit{Step 4:} Sample $\tau_{11}, \ \tau_{22}, \ \tau_{12}, \ \tau_{33 }$ from Normal distributions with parameters determined using \citet{Buice2000}. \\
    \textit{Step 5:} Compute $C_{\varepsilon_2} = (\beta + 1) / \beta$. \\
    \textit{Step 6:} Compute $C_{\mu} = \left(\dfrac{2 \tau_{12}}{\tau_{11}+\tau_{22}+\tau_{33}}\right)^2$. \\
    \textit{Step 7:} Compute  $C_{\varepsilon_1} =  C_{\varepsilon_2} - \dfrac{\kappa^2}{\sigma_{\varepsilon}\sqrt{C_{\mu}}}$. \\
\end{algorithm}

\section{Complete win rate analysis of the deductive approach benchmark}\label{app:deductive}

This appendix complements the analysis of Section~\ref{sec:constraint} with the complete win rate comparison for all three outputs.
Figure~\ref{fig:winrate_all} reports the win rate analysis at
$N \in \{20, 50, 100\}$. Three observations stand out. First, the winrates of the two deductive methods change across deduction scenarios at every training set size, while the CMOGP varies as a mechanical consequence of its fixed RMSE. The variability of the deductive methods across scenarios is therefore present at $N = 20$, $50$, and $100$, and
does not vanish as the training set grows. Second, the independent GP has a markedly low win rate in the middle panel of each row at the scenario where $f_2$ is the deduced output, with values ranging from 6\% at $N = 20$ to $0\%$ at $N = 100$ : this is consistent with the
shift of its RMSE distribution towards higher values when $f_2$ is deduced. Third, and most importantly, the CMOGP is the most frequent best method in nearly all configurations and across all evaluated outputs, with a clear margin over the two deductive methods.

\begin{figure}
  \centering
  \begin{subfigure}{\textwidth}
    \centering
    \includegraphics[width=\textwidth]{figure/stacked_bars_N20.pdf}
    \caption{$N = 20$}
  \end{subfigure}

  \vspace{0.2cm}
  \begin{subfigure}{\textwidth}
    \centering
    \includegraphics[width=\textwidth]{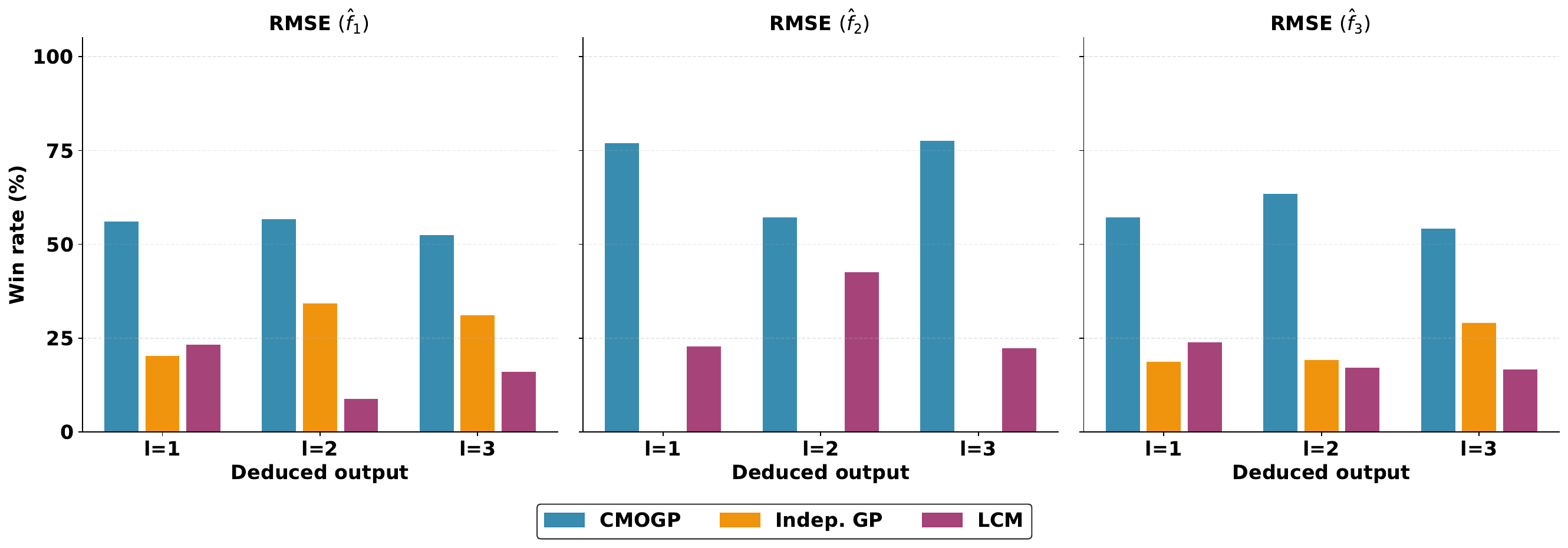}
    \caption{$N = 50$}
  \end{subfigure}

  \vspace{0.2cm}
  \begin{subfigure}{\textwidth}
    \centering
    \includegraphics[width=\textwidth]{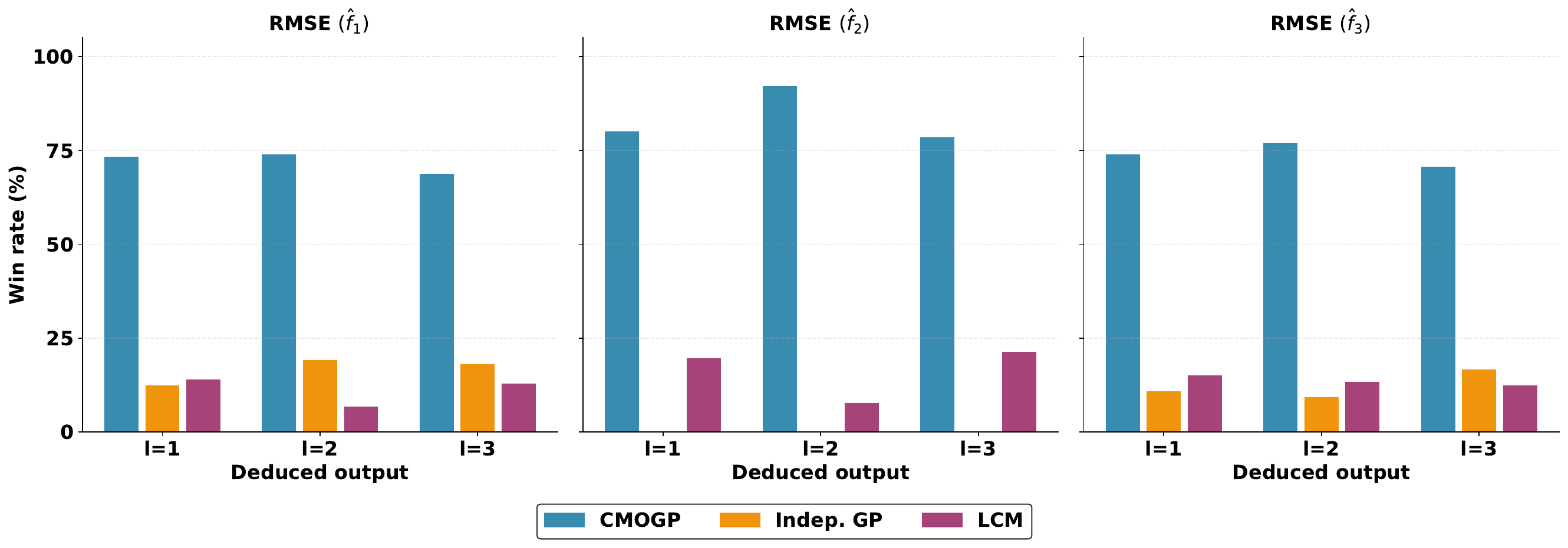}
    \caption{$N = 100$}
  \end{subfigure}
  \caption{Win rate (\%) of each method on the RRMSE of each output ($f_1$, $f_2$, $f_3$) for $N \in \{20, 50, 100\}$, across all deduction scenarios ($k_* \in \{1, 2, 3\}$). The highlighted column indicates the scenario where the evaluated output is the deduced one. The CMOGP (blue) is the same model in all scenarios. Results are aggregated over $200$ independent replications.}
  \label{fig:winrate_all}
\end{figure}
\clearpage

\subsection{Predictive intervals}\label{app:interval_len}

Figure~\ref{fig:interval_len_F2} shows the distribution of the average $90\%$ predictive interval length for $f_2$ across all deduction scenarios. For the deductive
methods, predictive intervals on the deduced output are obtained by propagating
the posterior samples of the modeled outputs through \eqref{eq:deduction}.
\begin{figure}
  \centering
  \includegraphics[width=\textwidth]{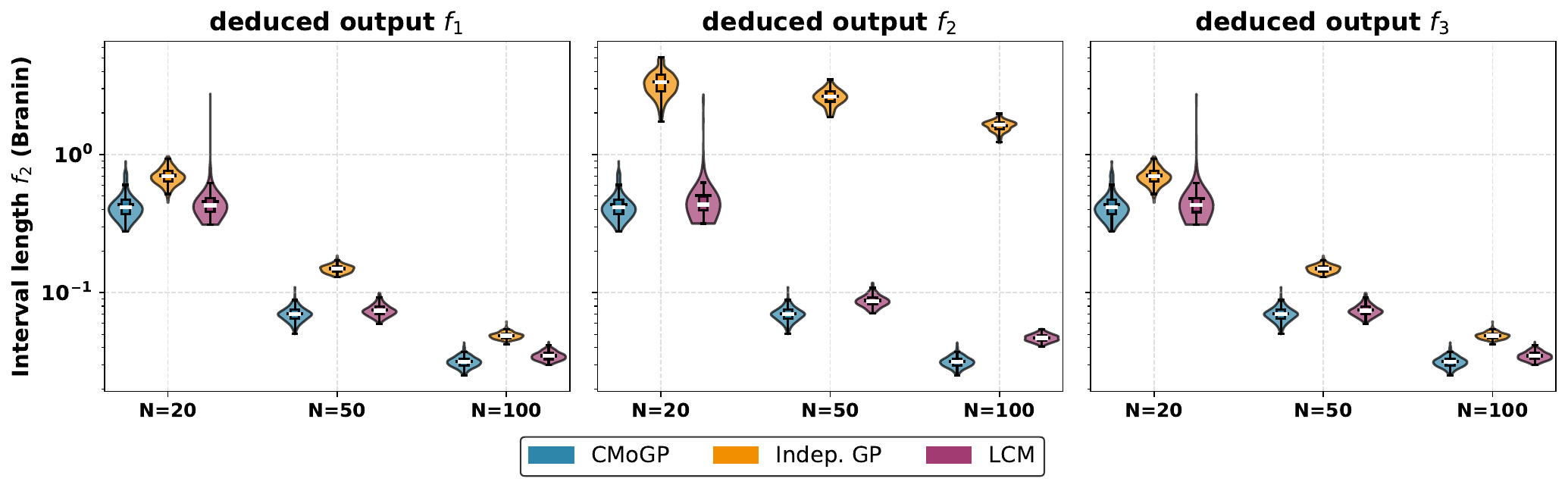}
  \caption{Distribution of the average $90\%$ predictive interval length for $f_2$ across 200
replications, for training set sizes $N \in {20, 50, 100}$ and the three deduction scenarios. All methods
achieve comparable empirical coverage rates. Shorter intervals thus indicate more informative predictions. The CMOGP maintains stable and shorter intervals across
all scenarios.}
\label{fig:interval_len_F2}
\end{figure}

\section{Complete RRMSE results for the Lotka--Volterra experiment}\label{app:lv_results}
Figure~\ref{fig:lv_all_outputs} displays the RRMSE as a function of the latent dimension $m$ for all four outputs ($p$, $q$, $r$, $s$) and the three training set sizes ($N \in \{30, 15, 10\}$), comparing \texttt{Row-CMO} against \texttt{Col-Indep} and \texttt{Fw-Indep} with all deduction scenarios. The patterns discussed in Section~\ref{sec:lv} on the output $p$ are consistently observed across all outputs: \texttt{Row-CMO} starts above the alternatives at small $m$ and surpasses them at moderate-to-high latent dimensions, with the gap widening as $N$ decreases.
 
\begin{figure}
  \centering
  \includegraphics[width=\textwidth]{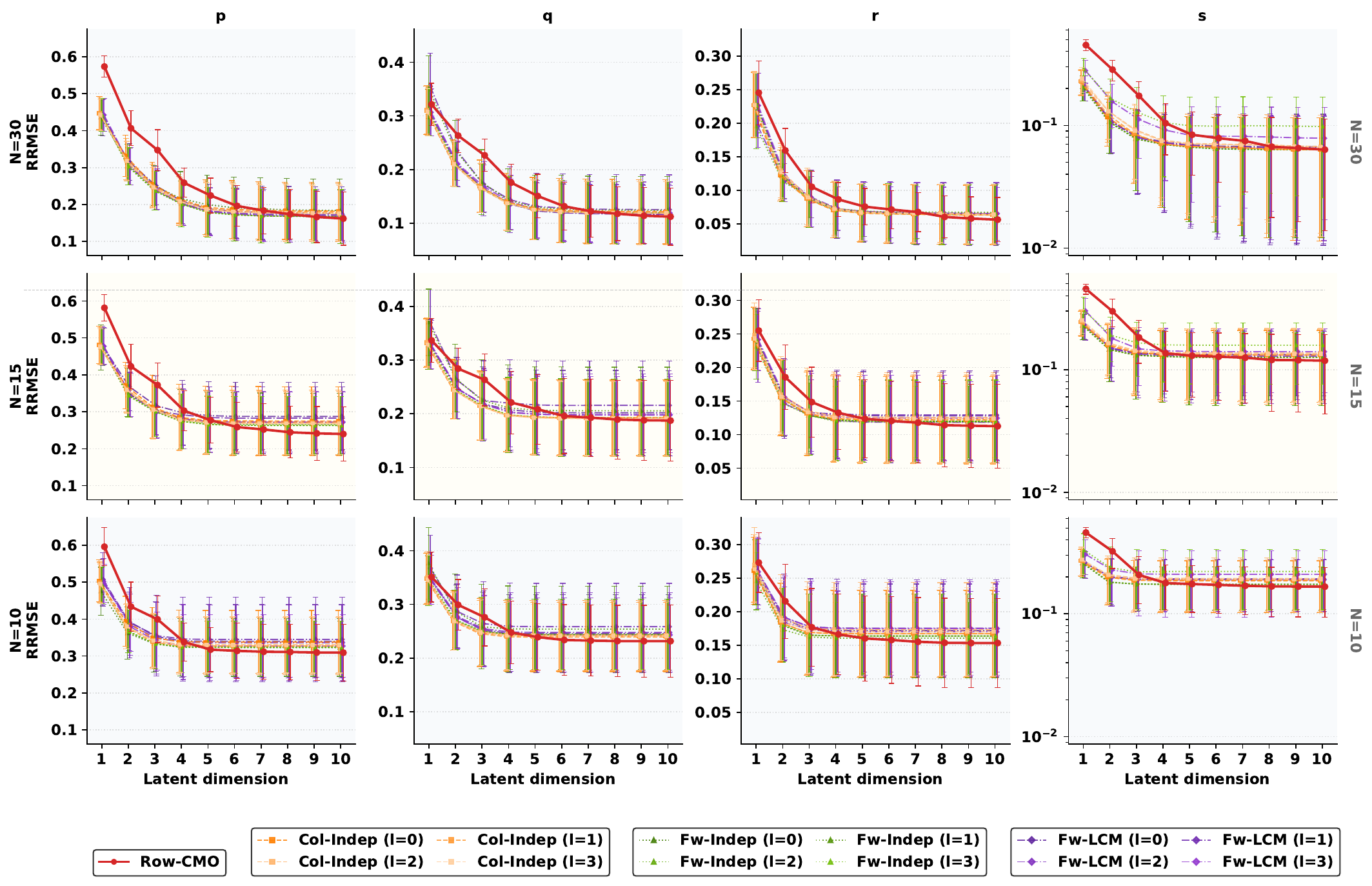}
  \caption{RRMSE as a function of $m$ for all four Lotka--Volterra outputs ($p$, $q$, $r$, $s$, columns) and training set sizes ($N = 30, 15, 10$, rows). \texttt{Row-CMO} (red) is compared against \texttt{Col-Indep} and \texttt{Fw-Indep} with all deduction scenarios $k_* \in \{0, 1, 2, 3\}$. Error bars represent $\pm 1$ standard deviation over $10$ seeds.}
  \label{fig:lv_all_outputs}
\end{figure}
\clearpage


\bibliographystyle{cas-model2-names}

\bibliography{cas-refs}



\end{document}